\documentclass{article} 
\usepackage[final]{colm2026_conference}
\usepackage{graphicx}
\usepackage{smilegroup}
\usepackage{enumitem}
\usepackage{hyperref}
\usepackage{url}
\usepackage{algorithm}
\usepackage{algorithmic}
\usepackage[most]{tcolorbox}
\usepackage{caption}
\definecolor{customblue}{HTML}{367DB0}
\definecolor{lightblue}{HTML}{a8dadc}
\definecolor{mydarkgreen}{RGB}{0,128,0}
\usepackage[textsize=tiny]{todonotes}

\providecommand{\query}{\mathsf{query}}
\providecommand{\prompt}{P}
\providecommand{\seqspace}{\mathcal{S}}
\providecommand{\calDx}{\mathcal{D}_x}
\providecommand{\calF}{\mathcal{F}}
\providecommand{\calH}{\mathcal{H}}
\providecommand{\calDH}{\mathcal{D}_{\mathcal{H}}}
\providecommand{\calX}{\mathcal{X}}
\providecommand{\calY}{\mathcal{Y}}
\providecommand{\calD}{\mathcal{D}}
\providecommand{\iid}{\stackrel{\mathrm{ i.i.d.}}{\sim}}
\providecommand{\R}{\mathbb{R}}
\providecommand{\N}{\mathbb{N}}
\providecommand{\E}{\mathbb{E}}

\ifx\theorem\undefined
  \newtheorem{theorem}{Theorem}[section]
\fi
\ifx\proposition\undefined
  \newtheorem{proposition}[theorem]{Proposition}
\fi
\ifx\lemma\undefined
  \newtheorem{lemma}[theorem]{Lemma}
\fi
\ifx\corollary\undefined
  \newtheorem{corollary}[theorem]{Corollary}
\fi
\theoremstyle{definition}
\ifx\definition\undefined
  \newtheorem{definition}[theorem]{Definition}
\fi
\ifx\assumption\undefined
  \newtheorem{assumption}[theorem]{Assumption}
\fi
\theoremstyle{remark}
\ifx\remark\undefined
  \newtheorem{remark}[theorem]{Remark}
\fi

\usepackage{microtype}
\usepackage{booktabs}
\usepackage{wrapfig}

\usepackage{lineno}

\definecolor{darkblue}{rgb}{0, 0, 0.5}
\hypersetup{colorlinks=true, citecolor=darkblue, linkcolor=darkblue, urlcolor=darkblue}

\title{FERA: Uncertainty-Aware Federated Reasoning for Large Language Models}

\author{Ruhan Wang \\
Indiana University\\
\texttt{ruhwang@iu.edu} \\
\And
Chengkai Huang \\
{The University of New South Wales} \\
\texttt{chengkai.huang1@unsw.edu.au} \\
\And
Zhiyong Wang \\
The Chinese University of Hong Kong \\
\texttt{zhiyongwangwzy@gmail.com} \\
\And
Junda Wu \\
{University of California San Diego}\\
\texttt{juw069@ucsd.edu}
\And
Rui Wang \\
Adobe Research \\
\texttt{ruiwan@adobe.com}
\And
Tong Yu \\
Adobe Research  \\
\texttt{tyu@adobe.com}
\And
Julian McAuley \\
{University of California San Diego}\\
\texttt{jmcauley@ucsd.edu}
\And
Lina Yao \\
{The University of New South Wales} \\
\texttt{lina.yao@unsw.edu.au}
\And
Dongruo Zhou\\
Indiana University\\
\texttt{dz13@iu.edu}
}

\begin{document}

\ifcolmsubmission
\linenumbers
\fi

\maketitle

\begin{abstract}
Large language models (LLMs) exhibit strong reasoning capabilities when guided by high-quality demonstrations, yet such data is often distributed across organizations that cannot centralize it due to regulatory, proprietary, or institutional constraints. We study \emph{federated reasoning}, where a server improves multi-step reasoning by coordinating with heterogeneous clients holding private demonstrations, without centralized training or raw data sharing. The key challenge is that client reliability is query-dependent, while the server cannot inspect client data to determine which contributions are trustworthy. To address this, we propose \emph{Uncertainty-Aware Federated Reasoning} (FERA), a training-free framework based on iterative server--client co-refinement. Across communication rounds, clients generate reasoning traces with lightweight uncertainty estimates, and the server synthesizes them into improved reasoning that is redistributed as context for the next round, progressively improving both server outputs and client-side reasoning. Within each round, \emph{Uncertainty-Aware Self-Critique Aggregation} (UA-SCA) resolves conflicts among heterogeneous client traces through query-dependent trust weighting and structured cross-client verification. Rather than simply discarding low-quality traces, UA-SCA revises flawed reasoning steps to recover useful information. We provide theoretical guarantees showing that the proposed iterative protocol converges, and that uncertainty-aware weighting accelerates convergence. Experiments on multiple reasoning benchmarks show that FERA consistently outperforms both federated training and training-free baselines, achieving progressively higher accuracy across rounds while maintaining communication and computational efficiency.
\end{abstract}

\section{Introduction}
Reasoning ability is fundamental to complex problem solving, as it enables a system to decompose high-level objectives into a coherent sequence of logical steps rather than relying on stored knowledge \citep{zhang2023cumulative,sun2023survey}. To enhance such reasoning capabilities, recent work has leveraged Large Language Models (LLMs) to generate explicit step-by-step deductive chains during problem solving \citep{wei2022chain,wang2023plan}. However, the reasoning performance of LLMs cannot be substantially improved by relying solely on their pre-trained capabilities, particularly for rigorous and structured tasks \citep{morishita2024enhancing,shojaee2025illusion}. Effective reasoning enhancement instead depends on access to high-quality data that can guide and refine the generation of logical steps \citep{morishita2024enhancing,cheng2025empowering}. In practice, such reasoning data are difficult to obtain, as they are often scattered across diverse domains and subject to privacy or proprietary constraints \citep{green2025leaky,rischke2022federated,baack2025towards}. As a result, a central research challenge is how to improve reasoning performance under limited data availability. For example, in a healthcare network where hospitals hold proprietary clinical data, a central server cannot pool data due to privacy regulations but can aggregate reasoning traces generated locally by each hospital's LLM. We refer to this setting as \emph{federated reasoning}.

Federated reasoning provides a way to improve LLM reasoning by allowing collaboration across decentralized data sources while keeping raw data local \citep{liu2024deepseek, achiam2023gpt,wei2022chain, kojima2022large}. Existing methods generally fall into training-based and training-free approaches. Training-based methods require clients to fine-tune local models and exchange parameters \citep{wang2024flora, wu2024fedbiot, ma2023fedid}, which is costly for modern LLMs in both computation and communication~\citep{yan2025federated, shu2024ferret, liu2023federated, che2023federated, wang2024one, liu2025ecolora}. Training-free methods reduce this cost by exchanging only lightweight information, such as prompts or retrieved examples \citep{chen2025can, wang2025federated, wu2024federated, du2023improving}. However, their performance is strongly influenced by cross-client heterogeneity. Under domain mismatch, local LLMs may produce overconfident yet systematically biased reasoning, leading the server to receive conflicting responses to the same query. Because each client reasons over \emph{private, non-overlapping} data, such disagreements reflect genuine information asymmetry rather than mere noise. However, the server cannot verify which client’s knowledge is most relevant to a given query. This challenge motivates the central question we address:

\begin{tcolorbox}[colback=blue!3!white, colframe=blue!50!black, boxrule=0.4pt, left=4pt, right=4pt, top=3pt, bottom=3pt]
How can a training-free federated reasoning framework progressively improve reasoning quality across rounds while reliably aggregating conflicting client contributions with query-dependent reliability?
\end{tcolorbox}

To address this question, we propose \textbf{Uncertainty-Aware Federated Reasoning (FERA)}. Our key contributions are:

\begin{figure*}[t]
    \centering
    \includegraphics[width=\linewidth]{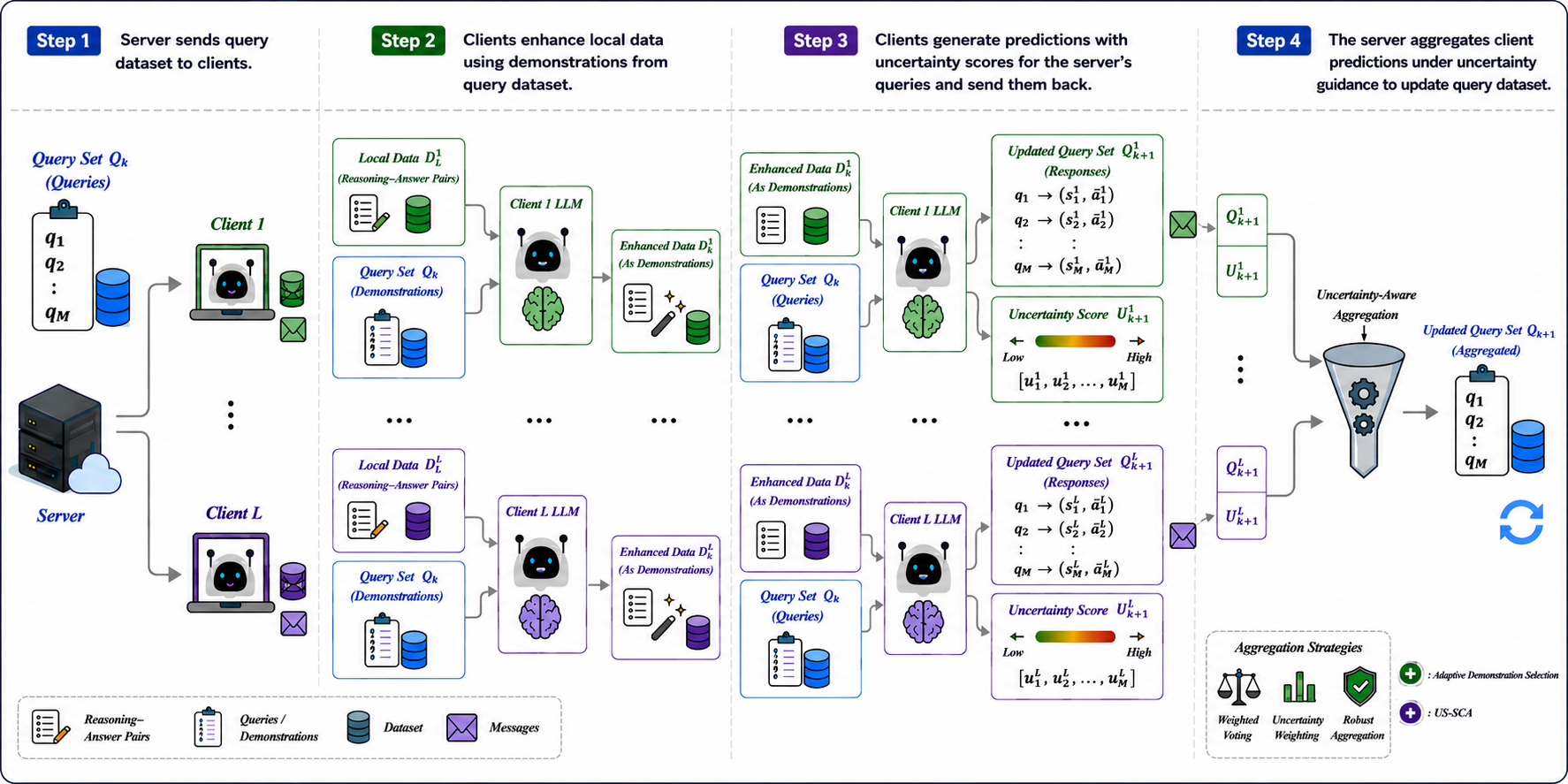}
    \vspace{-0.2in}
    \caption{\footnotesize Overview of the FERA framework. Over multiple rounds, the server distributes context to clients, who generate reasoning traces with uncertainty estimates from private data. The server refines its reasoning via \textbf{\textit{UA-SCA}} and redistributes improved context, creating a co-refinement loop where both server outputs and client contexts improve simultaneously. Detailed explanations are provided in Section~\ref{sec:aggregation}.}
    \label{fig:FERA}
\end{figure*}

\begin{itemize}[leftmargin=*, nosep]
\item We introduce FERA, a training-free \emph{iterative co-refinement} framework for federated reasoning (Figure~\ref{fig:FERA}). FERA operates over multiple communication rounds in which the server distributes its current reasoning context to clients, each client generates reasoning traces with uncertainty estimates using its private data, and the server synthesizes improved reasoning that is redistributed as enhanced context for the next round. This bidirectional loop allows both the server’s answers and clients’ local contexts to improve progressively, without parameter updates or data sharing.

\item To resolve conflicting reasoning within each round, we introduce \textbf{Uncertainty-Aware Self-Critique Aggregation (UA-SCA)}, which addresses the query-dependent reliability problem at the reasoning level rather than only at the answer level. UA-SCA assigns query-dependent trust weights based on token-level entropy, and employs structured self-critique where each client’s reasoning is cross-examined against competing answer groups at the server. This allows flawed reasoning steps to be identified and \emph{revised} using evidence from other clients, rather than simply discarded. Ablation studies (Section~\ref{sec:4.4}) confirm that uncertainty weighting and self-critique provide complementary gains.

\item We provide theoretical analysis under a simplified \emph{linear self-attention} model~\citep{zhang2023trained}, proving that FERA’s iterative protocol converges to ground-truth answers as the number of demonstrations grows. Importantly,  incorporating uncertainty-aware weights provably accelerates the convergence rate, providing principled motivation for the framework design.

\item We evaluate FERA on general and mathematical reasoning benchmarks under heterogeneous client data. FERA consistently outperforms both training-based and training-free baselines, with accuracy improving across communication rounds, while maintaining significantly higher computation and communication efficiency than other baselines.
\end{itemize}

\section{Related Work}
Recent advances in FL and LLMs increasingly highlight uncertainty as a key tool for handling heterogeneity and improving decision reliability (see Appendix~\ref{app:related_work} for details). In federated reasoning, existing methods can be broadly categorized into \emph{training-based} and \emph{training-free} approaches. Training-based methods, such as FedKSeed~\citep{qin2023federated} and FLoRA~\citep{wang2024flora}, improve performance via parameter updates but incur additional training cost and system complexity. In contrast, training-free methods~\citep{liu2023federated, chen2025can, wang2025federated, wu2024federated, du2023improving} are more lightweight but struggle with data heterogeneity and complex reasoning tasks. Among them, prompt aggregation methods~\citep{liu2023federated, chen2025can} often rely on data homogeneity or retrieval quality, in-context learning approaches~\citep{wang2025federated, wu2024federated} are typically limited to simple QA settings, and debate-based frameworks~\citep{du2023improving} require direct inter-client communication and treat all client outputs as equally trustworthy. Meanwhile, uncertainty quantification has shown broad effectiveness in FL for handling heterogeneity and model differences~\citep{koutsoubis2025privacy, zhang2025uncertainty, wang2024bridging, zhang2024uncertainty}, and in LLMs for uncertainty decomposition, active learning, and adaptive reasoning~\citep{hou2023decomposing, huang2024unlocking, wang2025uncertainty, yang2023improving}. These advances motivate FERA, which leverages uncertainty to enable training-free federated reasoning under heterogeneous client knowledge.
\section{Preliminaries}
\label{sec:Preliminaries}
\vspace{-0.3em}
We consider a federated reasoning setting in which a central server coordinates $L$ clients under a standard client--server architecture. We first define the notation used throughout this paper.

\textbf{Notation.} A \emph{query} $q$ is a question to be answered. A \emph{reasoning trace} $s_{\{1:T\}} = (s_1, \ldots, s_T)$ is a sequence of $T$ intermediate reasoning steps leading to a final answer $a$. A \emph{demonstration} is a complete triple $(q, s_{\{1:T\}}, a)$ used as a prompt exemplar to guide LLM generation via in-context learning.
\vspace{-0.4em}

\textbf{Client-side data.} Each client $i$ holds a private dataset $D^i = \{(q_n^i, s_{\{1:T\},n}^i, a_n^i)\}_{n=1}^N$ of $N$ demonstrations, representing the client's domain-specific reasoning knowledge. These demonstrations are never shared with the server; only the reasoning outputs generated from them are communicated.
\vspace{-0.4em}

\textbf{Server-side data.} The server maintains a query set $Q = \{(q_m, s_{\{1:T\},m}, a_m)\}_{m=1}^M$ of $M$ entries, representing the questions the server seeks to answer.

\noindent In this setting, the server must coordinate with clients whose data quality and domain coverage vary from query to query. Since client data is private and heterogeneous, client-generated reasoning traces for the same query may conflict or contain overconfident errors. This raises several questions:
\begin{itemize}[leftmargin=*, nosep]
    \item How can the server determine which clients are reliable for a given query without observing their private data?
    \item How can the server identify and correct flawed intermediate reasoning steps when clients produce conflicting traces for the same query, rather than simply picking one answer?
    \item How can this be achieved efficiently in both computation and communication, without model fine-tuning or parameter exchange?
\end{itemize}
\vspace{-0.3em}
In the following sections, we address these challenges and present \textbf{FERA}, a training-free framework for uncertainty-aware federated reasoning.
\begin{algorithm*}[t]
\caption{Uncertainty-Aware Self-Critique Aggregation}
\label{alg:UA-SCA}
\begin{algorithmic}[1]
\REQUIRE Query index $m$, round $k$, client submissions $\{(s_{\{1:T\},k, m}^i, a_{k,m}^i, u_{k,m}^i)\}_{i=1}^L$, the server LLM denoted by $\text{LLM}^S$.
\ENSURE Aggregated reasoning--answer pair $(s^\star, a^\star)$.
\STATE Partition $\{(s_{\{1:T\},k, m}^i, a_{k,m}^i)\}_{i=1}^L$ into groups $\mathcal{G}=\{G\}$ based on $a_m$.
\FOR{each group $G \in \mathcal{G}$}
  \STATE $S_{G} \leftarrow \textsf{Summarize}(\{s_{\{1:T\},k, m}^i : (s_{\{1:T\},k, m}^i, a_{k,m}^i)\in G\};\,\text{LLM}^S)$.
\ENDFOR
\FOR{$i=1,\dots, L$}
\STATE Denote $G(i) \in \cG$ as the group which $(s_{\{1:T\},k, m}^i, a_{k,m}^i)$ belongs to.
  \STATE $(\hat s_{\{1:T\},k, m}^i, \hat a_{k,m}^i) \leftarrow \textsf{SelfCritique}((s_{\{1:T\},k, m}^i, a_{k,m}^i), \{S_{G}: G \neq G(i)\};\,\text{LLM}^S)$.
\ENDFOR
\STATE Calculate weights ${\{w_i\}}_{i=1}^{L}$ based on \ref{eq:www}. 
\STATE $(s_{\{1:T\}, k+1, m}, a_{\{1:T\}, k+1, m}) \leftarrow \textsf{Aggregate}(\{(\hat s_{\{1:T\},k, m}^i, \hat a_{k,m}^i, w_i)\}_{i=1}^L;\,\text{LLM}^S)$ .
\STATE \textbf{return} $(s_{\{1:T\}, k+1, m}, a_{\{1:T\}, k+1, m})$.
\end{algorithmic}
\end{algorithm*}

\section{Uncertainty-Aware Federated Reasoning}
\label{sec:3}

We now present FERA. Two observations motivate its design. First, a single round of client responses is often noisy and incomplete, especially under heterogeneous data; iterating over multiple rounds allows both the server's reasoning and clients' local contexts to progressively improve. Second, within each round, client reliability varies from query to query, and the server has no direct way to judge which responses to trust. FERA combines an iterative server--client co-refinement protocol (Section~\ref{sec:iterative}) with uncertainty-guided reasoning correction (Section~\ref{sec:uasca}) to address both challenges jointly.

\subsection{Iterative Server--Client Co-Refinement}
\label{sec:iterative}

FERA operates over $K$ communication rounds (Algorithm~\ref{alg:FERA}). Each round consists of three steps. First, clients enrich their private datasets using server-distributed context as demonstrations (\emph{local refinement}). Second, clients generate reasoning--answer pairs for the server's queries using the enriched datasets, together with uncertainty estimates (\emph{client labeling}). Third, the server synthesizes improved reasoning from client contributions via UA-SCA and redistributes the updated context for the next round (\emph{server update}). Only reasoning outputs and uncertainty signals are transmitted, keeping communication overhead low.

\textbf{Server Distribution.} At round $k$, the server distributes the query dataset $Q_k$ to all clients.
\vspace{-0.4em}

\textbf{Local Refinement.} Client $i$ selects a subset of demonstrations $\mathcal{S}_{k,Q}^i \subseteq Q_k$ from the server-provided query set and incorporates them into prompts for its local language model $\mathrm{LLM}^i$. Using these demonstrations, the client generates refined reasoning--answer pairs
$(s_{\{1:T\},k,n}^i, a_{k,n}^i) \sim \text{LLM}^i(s_{\{1:T\}}, a \mid q_{n}^i, \mathcal{S}_{k,Q}^i)$,which are appended to the local dataset, yielding the updated dataset
$D_k^i = D^i \cup \{(q_{n}^i, s_{\{1:T\},k,n}^i, a_{k,n}^i)\}_{n=1}^{N}$. The demonstration set $\mathcal{S}_{k,Q}^i$ is selected using Maximal Marginal Relevance (MMR)~\citep{carbonell1998use}:
\begin{small}
\begin{equation}
\mathcal{S}_{k,Q}^i
= \arg\max_{\mathcal{S} \subseteq Q_k} \Big[
\tfrac{1}{|\mathcal{S}|}
\textstyle\sum_{(q', s_{1:T}', a') \in \mathcal{S}}
Sim\big((q_n^i, s_{1:T,k-1,n}^i, a_{k-1,n}^i),
(q', s_{1:T}', a')\big)
- \lambda \cdot Div(\mathcal{S})
\Big],
\label{eq:mmr}
\end{equation}
\end{small}where $Sim(\cdot,\cdot)$ measures semantic similarity between examples, $\lambda > 0$ is the regularization parameter, and $Div(\mathcal{S})$ is a diversity measure of set $\mathcal{S}$.
\vspace{-0.4em}

\textbf{Client Labeling.} Using the refined dataset $D_k^i$, client $i$ generates predictions $\{a_{k+1,m}^i\}_{m=1}^M$ for the server queries $\{q_m\}_{m=1}^M$. For each query $q_m$, the client constructs a demonstration set $\mathcal{S}_{k,D}^i \subseteq D_k^i$ and prompts its local model to produce an updated reasoning--answer pair $(s_{\{1:T\},k+1,m}^i, a_{k+1,m}^i) \sim \mathrm{LLM}^i(s_{\{1:T\}}, a \mid q_m, \mathcal{S}_{k,D}^i)$. The demonstration set $\mathcal{S}_{k,D}^i$ is again selected using the MMR strategy to balance relevance to the current query and diversity among examples. Along with each generated response, the client computes an uncertainty score $u_{k+1,m}^i$ based on the output logits of $\mathrm{LLM}^i$.
\vspace{-0.4em}

\textbf{Server Update.} The server collects all client responses and updates $Q_{k+1}$ via UA-SCA (Section~\ref{sec:uasca}), then redistributes the refined $Q_{k+1}$ to clients for the next round. As $Q_{k+1}$ contains higher-quality reasoning traces than $Q_k$, it provides stronger demonstrations that enable clients to produce more accurate reasoning in round $k{+}1$. Meanwhile, each client augments its local dataset $D_k^i$ with reasoning--answer pairs derived from these refined traces. This process leads to progressive co-improvement of both server outputs and client-side context across rounds, while all communication remains server-mediated, ensuring that raw client data never leaves the local device.
\vspace{-0.4em}

\textbf{Privacy considerations.} FERA operates under a \emph{data-local} threat model: raw client data $D^i$ never leaves the client, and only reasoning traces generated for server-provided queries are transmitted. This is weaker than formal differential privacy but consistent with practical federated deployments where the primary concern is preventing raw data centralization. We acknowledge that reasoning traces may indirectly leak information about client data; an empirical leakage audit (Appendix~\ref{app:privacy}) using NER-based detection shows that only 0.043\% of transmitted responses contain identifiers not already present in the server query, indicating minimal unintended leakage in practice.

\subsection{Uncertainty-Aware Self-Critique Aggregation}
\label{sec:uasca}

Within each round, the server receives multiple reasoning traces for the same query, often with conflicting intermediate steps and final answers. The core challenge is that, without access to client data distributions, the server cannot determine which client is most trustworthy for the query at hand. Existing approaches are insufficient. Uniform aggregation~\citep{chen2025can} ignores query-dependent reliability. Iterative mutual alignment~\citep{du2023improving} may converge to a shared yet biased reasoning path because it lacks a mechanism for identifying trustworthy clients on a per-query basis. Uncertainty-weighted voting~\citep{agrawal2025uncertainty} uses confidence signals, but only at the answer level, and thus cannot correct flawed intermediate reasoning when a client is confidently wrong under domain mismatch. To address these limitations, we propose Uncertainty-Aware Self-Critique Aggregation (UA-SCA), which combines query-dependent uncertainty weighting with structured self-critique and cross-client verification. Instead of merely selecting or discarding entire responses, UA-SCA revises weakly justified but informative reasoning traces at the step level. The full three-stage pipeline is described in Algorithm~\ref{alg:UA-SCA}.

\textbf{Stage 1: Per-query uncertainty estimation.} Since the server cannot observe client data, each client self-reports an uncertainty score derived from output logits. Specifically, for a generated token at position $t$, the entropy is $H_t = -\sum_{i=1}^V p_{t,i}\log(p_{t,i}+\varepsilon)$, where $p_{t,i}$ is the softmax probability of token $i$, $V$ is the vocabulary size, and $\varepsilon$ is a small constant for numerical stability \citep{duan2023shifting, farquhar2024detecting, zhang2025token}. The uncertainty of a reasoning sequence is the average entropy over all $T$ generated tokens: $U=\frac{1}{T}\sum_{t=1}^T H_t$. This measure is training-free and calibration-free, making it well suited to the federated setting where clients deploy their own local models with accessible output logits. When output logits are not accessible, alternative uncertainty estimation methods such as verbalized confidence or sampling-based consistency can be used as replacements. A detailed discussion of design considerations for the uncertainty measure, including comparison with alternative methods, is provided in Appendix~\ref{app:unc_design}. The server converts uncertainty scores into query-dependent weights: $w_{k+1,m}^i = \frac{\exp(-u_{k+1,m}^i / \tau)}{\sum_{j=1}^L \exp(-u_{k+1,m}^j / \tau)}, \quad \tau > 0,\label{eq:weight}$ where $u^i$ is the uncertainty score of client $i$ and $L$ is the number of clients. The softmax in Eq.~\eqref{eq:weight} normalizes across clients for each query, so absolute uncertainty scales need not be comparable across different model families; only the relative ordering within a query matters. We empirically verify this under model-capacity heterogeneity (Qwen3-4B/1.7B/0.6B) in Appendix~\ref{app:exp_analysis}, where FERA maintains stable performance despite substantial differences in raw entropy scales. 


\textbf{Stage 2: Structured self-critique.} Uncertainty weighting helps identify \emph{which} clients are more trustworthy for a given query, but it cannot determine \emph{what} is wrong in their reasoning. Under domain mismatch, a client may still produce an incorrect answer with high confidence, as the model is confidently relying on incorrect domain knowledge. Therefore, beyond trust weighting, the server must also diagnose and revise flawed reasoning content at the step level. To enable this, UA-SCA first partitions client submissions into groups $\mathcal{G} = \{G_r\}$ according to exact matches of the final answer $a_{k,m}^i$ (e.g., the selected option in multiple-choice tasks or the numerical result in math tasks), where each group $G_r = \{i \mid a_{k,m}^i = r\}$ contains all clients supporting candidate answer $r$. For each group, UA-SCA constructs a representative summary of its reasoning traces, capturing the dominant reasoning pattern behind that answer. Each client's trace is then compared against summaries from \emph{other} groups and revised through structured self-critique when its intermediate steps conflict with explanations that are more coherent or better supported. In this way, UA-SCA targets unsupported assumptions and weak reasoning steps directly, allowing confident but poorly justified traces to be corrected using evidence from other clients rather than simply discarded. Although self-critique and cross-agent verification have been studied in centralized settings~\citep{yuan2025reinforce,xu2025stepwise,li2025dancing,du2023improving}, these methods typically assume full access to all reasoning traces and treat all agents as equally reliable. By contrast, UA-SCA is designed for the federated setting, where client data remains private and reliability is query-dependent, by coupling uncertainty-guided trust weighting with structured cross-client critique.

\textbf{Stage 3: Synthesis.} Finally, the \textsf{Aggregate} function (Algorithm~\ref{alg:UA-SCA}, line~9) instructs the server-side LLM to synthesize the revised traces into a single coherent reasoning path and final answer. The server LLM receives all revised client traces annotated with their uncertainty-derived weights $\{w_i\}$ and produces a unified output: $(s_{\{1:T\},k+1,m}, a_{k+1,m}) = \textsf{Aggregate}\big( \{(\hat{s}_{\{1:T\},k,m}^i, \hat{a}_{k,m}^i, w_i)\}_{i=1}^L \big)$. The server LLM preferentially follows higher-weighted (lower-uncertainty) clients while incorporating complementary information from others. Unlike traditional federated learning, which averages numerical parameters, this aggregation is a \emph{language-level synthesis} operation performed entirely by the server LLM. Importantly, the server LLM acts as a synthesis tool whose role is to integrate client contributions rather than to reason independently; the objective is not to preserve a verbatim record of individual client reasoning traces, but to improve server-side prediction quality by synthesizing diverse and potentially conflicting reasoning signals into a more coherent and broadly supported rationale. An illustrative example of the full UA-SCA procedure is provided in Appendix~\ref{app:ua_sca_example}.
\begin{tcolorbox}[colback=blue!3!white, colframe=blue!50!black, boxrule=0.4pt, left=4pt, right=4pt, top=3pt, bottom=3pt]
\textbf{Robustness.} UA-SCA is robust to unreliable and low-quality clients by design: uncertainty weighting down-weights noisy contributions (formally justified by Theorem~\ref{thm:1}), while structured self-critique detects and revises flawed reasoning independently of the reported confidence. This provides defense against non-strategic noise (e.g., domain mismatch, incomplete data); robustness to strategic adversaries that deliberately manipulate uncertainty scores would require additional adversarial mechanisms and is left as future work.
\end{tcolorbox}

\label{sec:aggregation}

\begin{figure*}[t]
    \begin{minipage}{\linewidth}
        \centering
        \includegraphics[width=\linewidth]{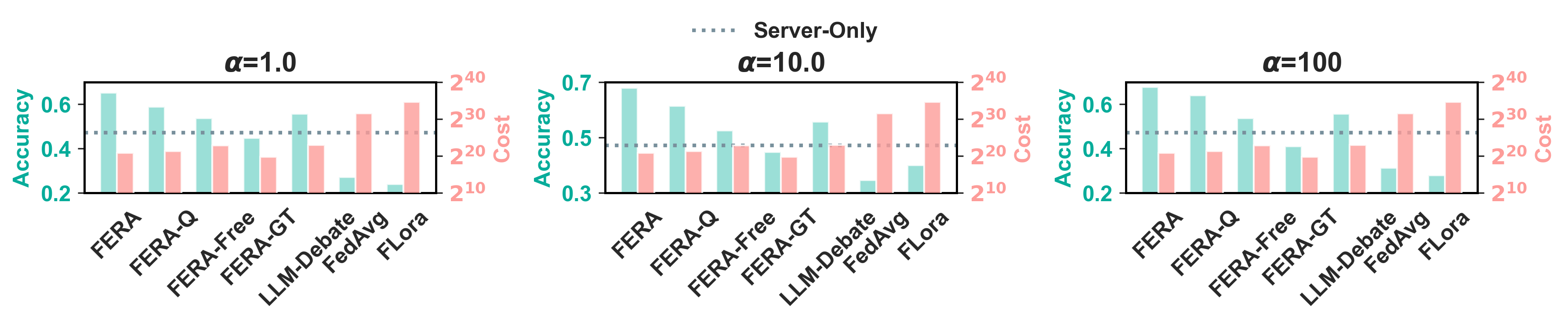}
        \vspace{-0.3in}
        \captionof{figure}{\footnotesize Performance comparison of \textsc{FERA} and its variants against baseline methods on the \textsc{MMLU-Pro} benchmark under varying levels of client-level data heterogeneity ($\alpha \in \{1.0, 10, 100\}$).}
        \label{fig:main_result}
    \end{minipage}
    \begin{minipage}{0.49\textwidth}
    \centering
   \includegraphics[width=\linewidth]{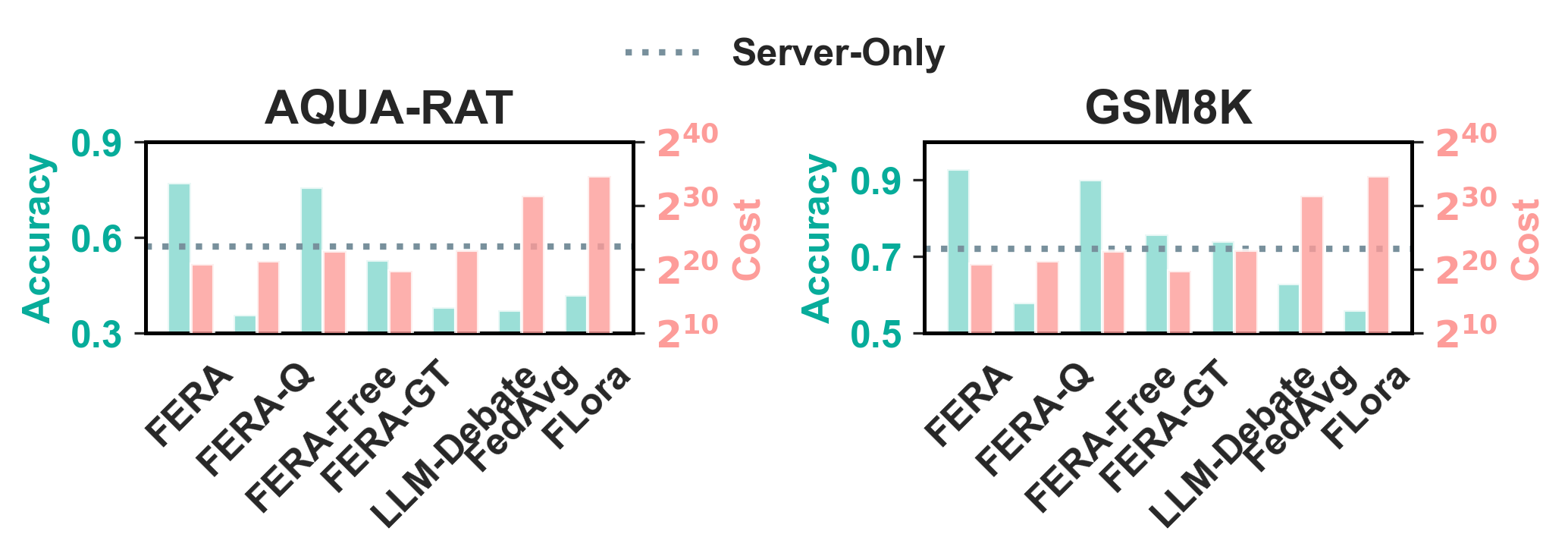}
   \vspace{-0.3in}
    \caption{\footnotesize Performance comparison of FERA and its variants against other baselines across two benchmarks: AQUA-RAT and GSM8K.}
    \label{fig:Main_results_2}
\end{minipage}
    \centering
    \begin{minipage}{0.49\textwidth}
        \centering
        \includegraphics[width=\linewidth]{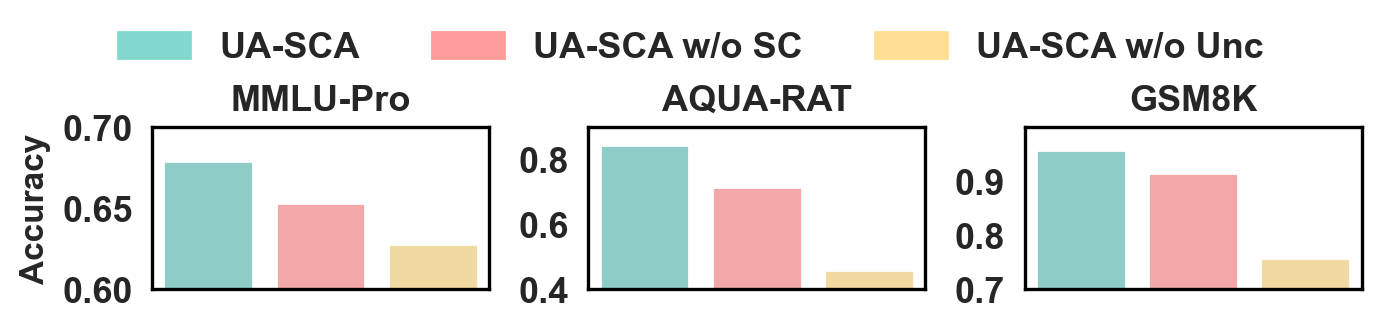}
        \vspace{-0.1in}
        \caption{\footnotesize Performance of FERA under different aggregation strategies across different benchmarks.}
        \label{fig:abl_aggre}
    \end{minipage}
    \hfill
    \vspace{-0.1in}
\end{figure*}
\subsection{Theoretical Analysis}
\label{sec:theory}

The notation used below is defined in Section~\ref{sec:Preliminaries}. While the iterative protocol and uncertainty weighting in FERA are motivated by practical considerations, it is natural to ask whether they admit formal justification. To address this question, we analyze Algorithm~\ref{alg:FERA} under a simplified setting that captures the core aggregation dynamics. Specifically, we consider a setting where the LLM predicts only the final answer, without intermediate reasoning steps, for linear regression tasks with $q \in \RR^d$ and $a \in \RR$. Following \citet{zhang2023trained}, each client's LLM is modeled as a single-layer linear self-attention (LSA) model (see Appendix~\ref{app:theory} for details). We assume $q_m, q_n^i \sim N(0, \Lambda)$ with $\Lambda \succ 0$, and that client $i$'s answers satisfy $a_n^i = (q_n^i)^\top \theta + \epsilon_n^i$, where $\epsilon_n^i \sim N(0, \sigma_i^2)$. The heterogeneous noise variances $\sigma_i^2$ capture variation in client reliability. Under this model, the server aggregates client predictions via weighted averaging: $a_{k+1,m} = \sum_{i=1}^L w_{k+1,m}^i a_{k+1,m}^i$.

\begin{theorem}\label{thm:1}
Set the weights $w_{k,m}^i: = w_m^i$ to be weights independent of the rounds and the initial answers $a_{0, m} = 0$. Then for Algorithm \ref{alg:FERA}, at every round $k \in [K]$, we have:
\vspace{-0.1in}
\begin{itemize}[leftmargin = *, nosep]
    \item $a_{k,m} = \theta_k^\top q_m$ for all $m \in [M]$.
    \item With probability at least $1-\delta$ for some $\delta\in (0,1)$, the difference between $\theta_k$ and $\theta$ can be bounded by
\begin{small}
\begin{equation}
\|\theta_k - \theta\|
\leq O\!\left(
\sqrt{\tfrac{d \log (L\delta^{-1})}{\min\{M,N\}}}
+ \sqrt{\tfrac{d \log(L\delta^{-1})}{N}}
+ \sqrt{\tfrac{d \log(L\delta^{-1})}{N}}
\, \lambda_{\min}^{-1/2}(\Lambda)
\sum_{i=1}^L w^i_{m} \sigma_i
\right).
\label{help:443}
\end{equation}
\end{small}

\end{itemize}
\end{theorem}

Theorem~\ref{thm:1} has two main implications. First, the iterative protocol converges: as $M,N$ increase, the server's answers approach the ground truth. Second, the error bound depends on $\sum_{i=1}^L w_m^i \sigma_i$, which is minimized when the weights are inversely proportional to $\sigma_i$. This aligns with Eq.~\eqref{eq:weight}, where clients with higher noise (and thus higher uncertainty) receive lower weights. The theorem therefore provides principled support for the uncertainty-aware weighting design under the simplified model. Although this analysis is derived under a tractable setting, two empirical findings suggest that its predictions extend more broadly. First, the ablation study in Section~\ref{sec:4.4} shows that removing uncertainty weighting consistently degrades performance across all benchmarks, consistent with the theorem's prediction that uniform weights ($w^i = 1/L$) are suboptimal. Second, the demonstration-quantity study in Appendix~\ref{app:exp_analysis} shows that increasing $N$ improves accuracy, matching the $O(1/\sqrt{N})$ convergence rate in Eq.~\eqref{help:443}.

\section{Experiment}
In this section, we first introduce the overall experimental setup. We then present a series of experiments designed to answer specific research questions, with each question and its corresponding results discussed in dedicated subsections.

\begin{itemize}[leftmargin=*, nosep]
\item \textbf{RQ1.} How do FERA and its variants perform on general and mathematical reasoning benchmarks under heterogeneous client data, including settings where clients have domain-specific expertise, compared with existing baselines?
\item \textbf{RQ2.} What is the effect of techniques such as Uncertainty-Aware Aggregation and Iterative Refinement on the performance of FERA?
\item \textbf{RQ3.} How do experimental factors such as the aggregation strategy, number of iterative update rounds, choice of backbone language model, specialized-domain settings, demonstration selection strategy, number of clients, presence of low-quality clients, and in-context length affect the performance of FERA?
\end{itemize}
\subsection{Experimental Setup}
\paragraph{Benchmarks.} We evaluate on three reasoning benchmarks: \textbf{MMLU-Pro} \citep{wang2024mmlu} (12K+ college-level questions across 14 topics), \textbf{AQUA-RAT} \citep{ling2017program} (100K+ algebraic word problems), and \textbf{GSM8K} \citep{cobbe2021training} (8.5K grade-school math problems). For MMLU-Pro, we select five questions per category as the server query set and partition the remainder among clients using a Dirichlet distribution \citep{hsu2019measuring} with $\alpha \in \{1.0, 10, 100\}$ to model varying heterogeneity levels. For AQUA-RAT and GSM8K, 70 problems form the server query set, and the rest are uniformly distributed among clients, providing a homogeneous data setting. Details are in Appendix~\ref{app:exp_setup}.

\textbf{Evaluation and Models.} We report accuracy as the primary metric \citep{team2023gemini,touvron2023llama} and computation/communication costs (Table~\ref{tab:cost_concrete}). Client-side models are Qwen3-4B \citep{yang2025qwen3} and LLaMA-3.1-8B \citep{dubey2024llama}; the server uses GPT-4o-mini for aggregation \citep{du2023improving}.

\subsection{Baselines}
\begin{wrapfigure}{r}{0.5\textwidth}
    \centering
    \vspace{-0.15in}
    \includegraphics[width=0.5\textwidth]{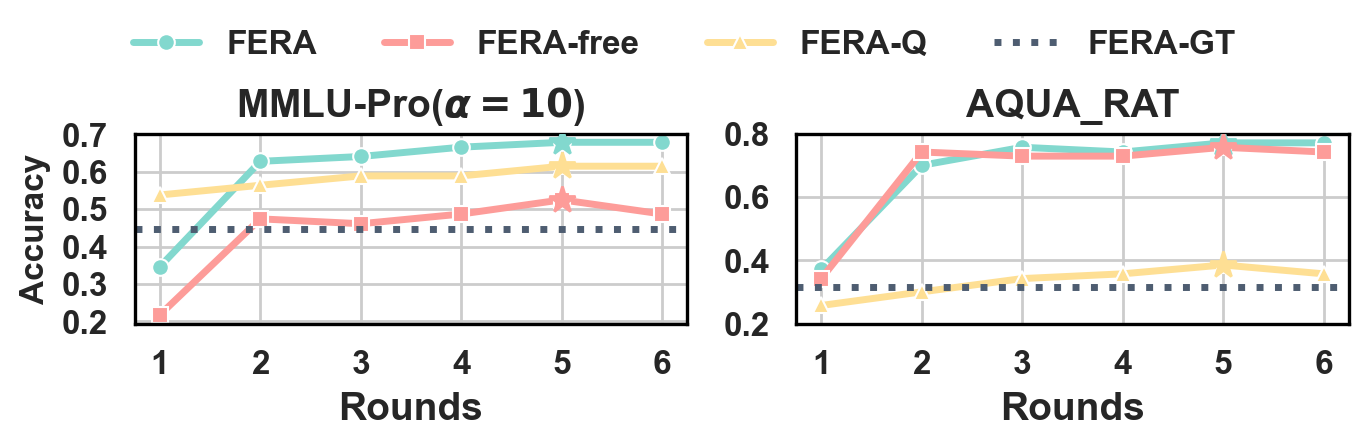}
    \vspace{-0.25in}
    \caption{\footnotesize Effect of interaction round count on FERA and FERA-Free in the MMLU-Pro benchmark.}
    \label{fig:abl_rounds}
    \vspace{-0.15in}
\end{wrapfigure}
\label{sec:baseline}
\textbf{External Baselines.} We evaluate FERA against two categories of baselines: federated learning (FL) methods and parameter-free methods. The FL baselines include \textbf{FedAvg} \citep{mcmahan2017communication,ye2024openfedllm} and \textbf{FloRA} \citep{wang2024flora}, while the parameter-free baseline is \textbf{LLM-Debate} \citep{du2023improving}. In addition, we consider a \textbf{Server-only} baseline, where the server LLM is directly applied to the benchmarks without any client collaboration.

\textbf{FERA Variants.} We introduce variants to isolate different components (details in Appendix~\ref{app:FERA_variants}): \textbf{FERA-GT} uses fixed ground-truth demonstrations without iterative updates (upper bound); \textbf{FERA-Q} omits intermediate reasoning steps to isolate their contribution; \textbf{FERA-Free} assumes clients have only questions without answers (low-resource setting).

\subsection{Evaluation Results}
\noindent The main results (Figures~\ref{fig:main_result}--\ref{fig:Main_results_2}) use LLaMA3.1-8B as the client model; Qwen3-4B results are in Appendix~\ref{app:exp_qwen}. FERA consistently outperforms all baselines with substantially lower communication overhead than training-based methods.

\textbf{Impact of Data Heterogeneity.} Figure~\ref{fig:main_result} shows that FERA consistently outperforms all baselines on MMLU-Pro across a wide range of heterogeneity levels. Using Dirichlet splits with $\alpha \in \{1.0, 10, 100\}$, we observe that performance improves as $\alpha$ increases, indicating that more homogeneous client distributions lead to more stable demonstration retrieval and fewer conflicting updates. FERA also performs strongly on AQUA-RAT and GSM8K under homogeneous settings (Figure~\ref{fig:Main_results_2}). We further evaluate an extreme heterogeneity setting in Appendix~\ref{app:spec_domain}.
\vspace{-0.4em}

\textbf{Impact of Iterative Refinement.}  Figure~\ref{fig:main_result} shows that FERA consistently improves with more interaction rounds, outperforming both Server-Only and FERA-GT (which uses fixed local context without iterative updates). This highlights the benefit of iterative refinement for integrating distributed knowledge.
\vspace{-0.4em}

\textbf{Impact of Reasoning Process.} FERA-Q performs substantially worse than FERA, especially on mathematically intensive benchmarks such as MMLU-Pro and AQUA-RAT. By restricting the model to final-answer prediction without intermediate reasoning steps, FERA-Q removes the structured problem-solving process needed for multi-step tasks, leading to lower accuracy and reduced interpretability. These results underscore the value of explicit reasoning traces in logic-intensive settings.
\vspace{-0.4em}

\textbf{Impact of Server LLM.} We verify that gains stem from framework design rather than the server model. The Server-Only baseline (same GPT-4o-mini, no clients) and LLM-Debate (same server for summarization) both underperform FERA, and the ``UA-SCA without Self-Critique'' ablation shows that even reducing the server to simple weighted synthesis still outperforms naive aggregation. 


\begin{wraptable}{r}{0.53\textwidth}
\centering
\vspace{-0.15in}
\caption{Total computation cost comparison.}
\label{tab:cost_concrete}
\footnotesize
\begin{tabular}{cccc}
\toprule
\textbf{Method} & \textbf{Rounds} & \textbf{FLOPs} & \textbf{Training?} \\
\midrule
FedAvg & 50 & $7.4 \times 10^{17}$ & Yes \\
FLoRA & 50 & $2.5 \times 10^{17}$ & Yes \\
LLM-Debate & 6 & $4.1 \times 10^{16}$ & No \\
FERA & 6 & $8.9 \times 10^{16}$ & No \\
\bottomrule
\end{tabular}
\vspace{-0.15in}
\end{wraptable}
\vspace{-0.4em}

\textbf{Computation and Communication Cost.} Using only forward inference over $K{=}6$ rounds, FERA requires $8\times$ fewer FLOPs than FedAvg and $3\times$ fewer than FLoRA (Table~\ref{tab:cost_concrete}). Although its FLOPs are slightly higher than LLM-Debate due to local refinement, this additional cost yields substantially better accuracy. Communication overhead is also low (Figures~\ref{fig:main_result}--\ref{fig:Main_results_2}), as FERA transmits only reasoning traces and uncertainty scores, without parameter exchange or additional client revision steps. Detailed cost formulas are provided in Appendix~\ref{app:cost}.

\subsection{Ablation Studies}\label{sec:4.4}
\textbf{Effect of Uncertainty-Aware Aggregation.} We compare three server-side aggregation strategies (Figure~\ref{fig:abl_aggre}): (i) UA-SCA; (ii) UA-SCA without Self-Critique; and (iii) UA-SCA without Uncertainty. UA-SCA consistently outperforms both ablations, confirming that uncertainty weighting (identifying \emph{which} clients to trust) and self-critique (correcting \emph{what} is wrong in reasoning) provide complementary gains. Even without self-critique, incorporating uncertainty yields clear improvements over naive aggregation, consistent with Theorem~\ref{thm:1}.

\textbf{Effect of Iterative Refinement.} Figure~\ref{fig:abl_rounds} compares FERA, FERA-Free, and FERA-GT across communication rounds. FERA and FERA-Free perform worse than FERA-GT in early rounds, as FERA-GT benefits from fixed, high-quality local exemplars. As rounds progress, the UA-SCA mechanism reduces noise in client updates, leading to consistent gains. This confirms that iterative refinement enables FERA to progressively integrate information from heterogeneous clients.

\vspace{0.2em}
\colorbox{yellow!30}{\parbox{\dimexpr\linewidth-2\fboxsep}{\textbf{Additional ablations on model capacity heterogeneity, uncertainty characteristics, demonstration quality, demonstration selection strategy, specialized domains, demonstration quantity, and the number of clients are reported in Appendix~\ref{app:exp_analysis}.}}}

\section{Conclusion}
We presented FERA, a training-free federated reasoning framework for improving LLM reasoning across heterogeneous clients with private data, without centralized training or data exchange. FERA combines iterative server--client co-refinement with UA-SCA, which uses uncertainty-aware weighting and structured cross-client verification to revise flawed reasoning. Theory shows that the iterative process converges and that uncertainty-aware weighting minimizes the error bound, while experiments demonstrate consistent gains over both training-based and training-free baselines on general and mathematical reasoning benchmarks. Important future directions include adversarial robustness and mixed-architecture deployments.

\bibliography{colm2026_conference}
\bibliographystyle{colm2026_conference}
\clearpage
\appendix

\begingroup
\linespread{1.3}\selectfont

\begin{center}
\textbf{\large Table of Contents for Appendix}
\end{center}
\vspace{0.3em}

\newcommand{\appsec}[3]{\noindent\hyperref[#1]{\textbf{#2\hspace{0.5em}#3}} \dotfill \pageref{#1}}
\newcommand{\appsubsec}[3]{\noindent\phantom{AA}\hyperref[#1]{#2\hspace{0.5em}#3} \dotfill \pageref{#1}}

\appsec{app:related_work}{A}{Additional Related Work}

\vspace{6pt}
\appsec{app:algorithm}{B}{Algorithm Details}\\[2pt]
\appsubsec{app:algorithm}{B.1}{Detailed FERA Workflow}\\[2pt]
\appsubsec{app:UA-WA}{B.2}{Uncertainty-Aware Aggregation}

\vspace{6pt}
\appsec{app:theory}{C}{Proof in Section~\ref{sec:3}}\\[2pt]
\appsubsec{app: pre}{C.1}{Preliminaries for the Theoretical Analysis}\\[2pt]
\appsubsec{app:thm1_proof}{C.2}{Proof of Theorem~\ref{thm:1}}

\vspace{6pt}
\appsec{app:prompt}{D}{Prompt}\\[2pt]
\appsubsec{app:server_init}{D.1}{Server Query Dataset Initialize}\\[2pt]
\appsubsec{app:client_resp}{D.2}{Client Response Generation for Server Queries}\\[2pt]
\appsubsec{app:ua_agg}{D.3}{Uncertainty-Aware Aggregation}

\vspace{6pt}
\appsec{app:exp_setup}{E}{Experiment Setup}\\[2pt]
\appsubsec{app:FERA_variants}{E.1}{FERA Variants}\\[2pt]
\appsubsec{app:client_data}{E.2}{Construction of Client Datasets}\\[2pt]
\appsubsec{app:impl}{E.3}{Implementation Details}\\[2pt]
\appsubsec{app:cost}{E.4}{Communication and Computational Cost}\\[2pt]
\appsubsec{app:demo_select}{E.5}{Demonstration Selection}\\[2pt]
\appsubsec{app:unc_calc}{E.6}{Uncertainty Calculation}

\vspace{6pt}
\appsec{app:exp_analysis}{F}{Supplementary Experiments}\\[2pt]
\appsubsec{app:exp_qwen}{F.1}{Main Results Using Qwen3-4B}\\[2pt]
\appsubsec{app:exp_analysis}{F.2}{Additional Ablation Study}

\vspace{6pt}
\appsec{app:unc_design}{G}{Design Considerations for the Uncertainty Measure}

\vspace{6pt}
\appsec{app:privacy}{H}{Privacy Analysis}

\vspace{6pt}
\appsec{app:case_study}{I}{Case Study}

\endgroup
\clearpage
\paragraph{LLM Usage Disclosure.} AI-Generated Visualizations.
The framework illustration in Figure 1 and several associated icons were generated or refined using LLM-assisted image generation tools (ChatGPT) for illustrative purposes only. These visual elements are used solely for presentation and do not influence the methodology, experiments, or conclusions of the paper.
\section{Additional Related Work}
\label{app:related_work}
\paragraph{Federated Learning in LLM.} 
The growing scale of LLMs raises concerns around computational demands and data privacy \citep{sani2024future}. FL offers a solution by enabling collaborative model adaptation across decentralized data sources without sharing raw data. Several works explore FL for LLM fine-tuning and instruction tuning. \citet{fan2023fate} apply FL to standard generation tasks, while \citet{wu2024fedbiot} introduce a privacy-preserving framework designed for secure, decentralized adaptation. \citet{kuang2024federatedscope} focus on instruction tuning across heterogeneous clients to improve generalization. Beyond fine-tuning, prompting-based strategies such as Fed-SP-SC and Fed-DP-CoT \citep{liu2023federated} have shown that reasoning capabilities can be improved via federated aggregation of chain-of-thought responses, leveraging self-consistency and diversity without model updates. Collectively, these efforts underscore FL's growing role in enabling scalable, privacy-conscious LLM development.

\paragraph{Reasoning Large Language Models.} 
Large language models have demonstrated remarkable capabilities in complex reasoning tasks through various prompting and inference strategies. Chain-of-thought (CoT) prompting revealed that LLMs possess a latent capacity for multi-step reasoning, leading to substantial improvements in arithmetic and commonsense tasks \citep{yoran2023answering,yao2023tree,chu2023navigate}. Follow-up studies extended this idea in several ways. Zero-shot CoT requires only the addition of “Let us think step by step” yet recovers much of the original benefit \citep{wei2022chain}. Self-consistency generates multiple reasoning paths and selects the most frequent answer among them \citep{wang2022self}. Least-to-most prompting decomposes difficult problems into a sequence of simpler subquestions that are solved in order \citep{zhou2022least}. ReAct interleaves internal thoughts with environment actions, thereby combining reasoning and tool use \citep{yao2023react}. More recently, \citet{shinn2023reflexion} enable models to iteratively critique and refine their own outputs, further enhancing reliability and accuracy. 

\paragraph{Federated Reasoning for Large Language Models.} 
Federated reasoning with large language models (LLMs) has been explored through training-driven approaches that update parameters in distributed settings. Full-parameter methods such as FedKSeed reduce bandwidth by transmitting random seeds and scalar gradients rather than full model weights \citep{qin2023federated}. Parameter-efficient methods include FLoRA, which aggregates heterogeneous low-rank adapters \citep{wang2024flora}, FedSP, which exchanges lightweight soft prompts while preserving server-side privacy \citep{dong2023tunable}, and FedBiOT, which splits models into emulator–adapter pairs optimized via bi-level frameworks \citep{wu2024fedbiot}. Task-specific frameworks have also been developed, including FedIT for instruction tuning \citep{zhang2024towards}, FedID for interactive distillation \citep{ma2023fedid}, and federated reinforcement learning from human feedback using lightweight preference selector aggregation. These advances show that federated training can preserve privacy and communication efficiency while remaining competitive with centralized training \citep{wei2025federated,wu2024towards}.

\paragraph{Training-free Approaches for LLM Reasoning.}  Training-free approaches aim to enhance reasoning without parameter updates. \citet{liu2023federated} leverage synonymous user questions with self-consistency voting and chain-of-thought prompting, though performance depends heavily on data homogeneity and retrieval quality. \citet{chen2025can} aggregate textual feedback into shared prompts using a density-based method, but struggle with misaligned prompts under heterogeneous data. Other in-context learning approaches \citep{wang2025federated,wu2024federated} remain limited to simple QA and tool-use tasks, while debate-style frameworks \citep{du2023improving} risk privacy leakage through direct client communication and underutilize local datasets. Together, these lines of work highlight a trade-off: training-driven methods achieve stronger performance at higher communication and optimization costs, whereas training-free methods are lightweight but less robust under heterogeneous data and complex reasoning, motivating the development of new frameworks such as FERA.

\paragraph{Uncertainty used in Federated Learning.} 
Recent work has addressed key challenges in FL through uncertainty-based approaches. To handle data heterogeneity, \citet{koutsoubis2025privacy} survey privacy-preserving methods that integrate uncertainty quantification in federated medical imaging, demonstrating how uncertainty metrics guide robust aggregation under non-IID conditions. Addressing the complementary challenge of model transparency, \citet{zhang2025uncertainty} propose UncertainXFL, which incorporates uncertainty-aware explanations directly into the aggregation process to enhance interpretability. Model heterogeneity presents another significant challenge that has been tackled through uncertainty-guided knowledge transfer mechanisms. \citet{wang2024bridging} develop FedType, which leverages proxy models with uncertainty-based distillation to bridge architectural differences, while \citet{zhang2024uncertainty} introduce UEFL, a dynamic approach that adapts discrete codebooks based on uncertainty measures to accommodate diverse data distributions across silos.

\paragraph{Uncertainty used in Large Language Models.} 
LLMs have focused on both quantifying and leveraging uncertainty for enhanced reliability. \citet{hou2023decomposing} pioneer the decomposition of aleatoric and epistemic uncertainty through ensemble methods over clarified input variants, achieving interpretable uncertainty estimates without requiring architectural modifications. This decomposition framework provides a foundation for more nuanced uncertainty-aware decision-making in downstream tasks. Building on this quantification approach, \citet{huang2024unlocking} demonstrate that output inconsistencies under label injection scenarios effectively capture intrinsic model uncertainty—a finding that directly enables active learning strategies for optimal in-context example selection. Further investigating uncertainty dynamics, \citet{wang2025uncertainty} reveal that increased exposure to in-context examples systematically reduces predictive uncertainty, with particularly pronounced effects on epistemic uncertainty. This reduction mechanism explains the observed improvements in both model confidence and accuracy as context size increases. Integrating these theoretical insights into practice, \citet{yang2023improving} operationalize an uncertainty-aware framework that empowers models with adaptive behavior: either self-correcting predictions when uncertainty is manageable or abstaining entirely when uncertainty exceeds predefined reliability thresholds.
\begin{algorithm*}[t]
\caption{Uncertainty-Aware Federated Reasoning}
\label{alg:FERA}
\begin{algorithmic}[1]
\REQUIRE 
$L$ clients, each with local dataset $D^i = \{(q_n^i, s_{\{1:T\},n}^i, a_n^i)\}_{n=1}^N$ and local model $\text{LLM}^i$. 
Server holds initial query set $\{q_m\}_{m=1}^M$.
\STATE Initialize server query set $Q_1 = \{(q_m, s_{\{1:T\},1,m}, a_{1,m})\}_{m=1}^M$.
\FOR{each round $k = 1,\dots,K$}
    \STATE \textit{Step 1:} Server distributes $Q_k$ to all clients. 
    \FOR{each client $i = 1,\dots,L$}
        \STATE \textit{Step 2:} 
        Refine local reasoning–answer pairs using demonstrations from $Q_k$. Form the enhanced dataset:
        \vspace{-0.4em}
        \[
        D_k^i \leftarrow D^i \cup \{(q_n^i, s_{\{1:T\},k,n}^{\prime i}, a_{k,n}^{\prime i})\}_{n=1}^N.
        \]
        \vspace{-1.5em}
        \STATE \textit{Step 3:} 
        Predict reasoning–answer pairs for server queries using demonstrations from $D_k^i$.  
        Compute uncertainty scores from the predictive distribution of $\text{LLM}^i$, and return to the server:
        \vspace{-0.4em}
        \[
        \{(s_{\{1:T\},k+1,m}^i, a_{k+1,m}^i, u_{k+1,m}^i)\}_{m=1}^M.
        \]
        \vspace{-1.8em}
    \ENDFOR
    \STATE \textit{Step 4:} 
    The server aggregates client responses using an uncertainty-aware scheme; see Section~\ref{sec:aggregation} for aggregation procedures under different algorithmic settings.

    \textbf{Update}
    \vspace{-0.4em}
    \[
    Q_{k+1} = \{(q_m, s_{\{1:T\},k+1,m}, a_{k+1,m})\}_{m=1}^M.
    \]
    \vspace{-1.5em}
\ENDFOR
\ENSURE Final predictions $\{(q_m, s_{\{1:T\},K+1,m}, a_{K+1,m})\}_{m=1}^M$.
\end{algorithmic}
\end{algorithm*}
\section{Algorithm Details}
\subsection{Detailed FERA Workflow}
\label{app:algorithm}
FERA follows a round-based federated reasoning workflow in which a central server iteratively refines its reasoning outputs through collaboration with multiple clients holding heterogeneous and private data. In each round, the server distributes its current query set to all clients, and each client uses local demonstrations to generate updated reasoning–answer pairs together with uncertainty estimates, without sharing raw data or model parameters. These client responses are then returned to the server and aggregated using an uncertainty-aware scheme to update the server’s reasoning results for the next round. Through this iterative process, FERA improves global reasoning quality by leveraging client-side knowledge and uncertainty signals, while remaining training-free, communication-efficient, and privacy-preserving. The overall workflow is summarized in Algorithm~\ref{alg:FERA}.

\begin{tcolorbox}[colback=blue!3!white, colframe=blue!50!black, boxrule=0.4pt, left=4pt, right=4pt, top=3pt, bottom=3pt]
\remark FERA does not inherently require strict synchronous participation. In asynchronous settings, Algorithm~\ref{alg:FERA} can be adapted by modifying line~8 to aggregate responses only from the subset of clients whose updates arrive within the current round. While a fully asynchronous variant would require a redesigned aggregation rule that jointly accounts for both uncertainty and update staleness, such an extension is orthogonal to the main focus of this work. Our goal is to introduce a training-free federated reasoning framework that leverages client-side data characteristics to guide server decision making.
\end{tcolorbox}

\subsection{Uncertainty-Aware Aggregation}
\subsubsection{Uncertainty-Aware Weighted Aggregation (UA-WA)}
\label{app:UA-WA}
\begin{figure}[h]
    \centering
    \includegraphics[width=0.5\linewidth]{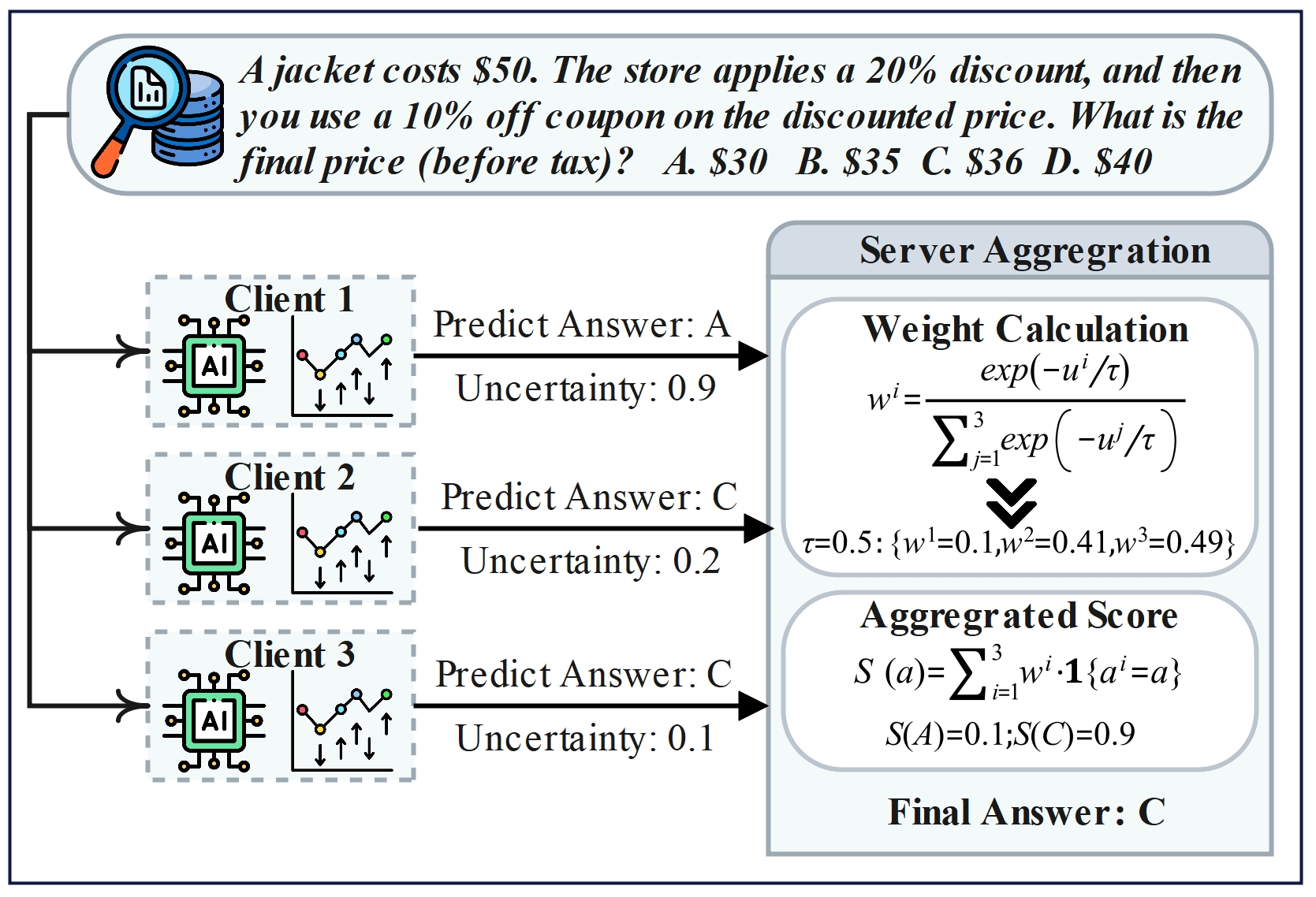}
    \caption{Illustration of UA-WA. For a given query, multiple clients generate candidate answers with associated uncertainty estimates. The server assigns higher weights to client responses with lower uncertainty and aggregates them accordingly, reducing the influence of unreliable or domain-mismatched predictions.}
    \label{fig:UA-WA}
\end{figure}

We first introduce the uncertainty-aware aggregation idea in a simplified setup where the reasoning steps $s_{\{1:T\}, k, m}^i$ are omitted and the task involves only questions $q$ and answers $a$. At round $k$, each of the $L$ clients uploads a predicted answer $a_{k+1,m}^i$ together with an associated uncertainty score $u_{k+1,m}^i \geq 0$ for query $q_m$. The server then calculates a weight for each client prediction using a temperature-scaled softmax over the negative uncertainties, so that predictions with lower uncertainty receive greater weight:
\begin{align}
    w_{k+1,m}^i = \frac{\exp(-u_{k+1,m}^i / \tau)}{\sum_{j=1}^L \exp(-u_{k+1,m}^j / \tau)}, 
\quad \tau > 0,\label{eq:www}
\end{align}
where $\tau$ is the temperature parameter. For each candidate answer $a \in \mathcal{A}$, the aggregated score is defined as $S_{k+1,m}(a) = \sum_{i=1}^L w_{k+1,m}^i \cdot \mathbf{1}\{a_{k+1,m}^i = a\}$, 
and the final prediction is selected as $a_{k+1,m} = \arg\max_{a \in \mathcal{A}} S_{k+1,m}(a)$.
The updated query--answer set is then $Q_{k+1} = \{(q_m, a_{k+1,m})\}_{m=1}^M$. Figure \ref{fig:UA-WA} illustrates the UA-WA process, where client predictions are weighted by their estimated uncertainty before aggregation.

\begin{tcolorbox}[colback=blue!3!white, colframe=blue!50!black, boxrule=0.4pt, left=4pt, right=4pt, top=3pt, bottom=3pt]
\remark The vanilla aggregation scheme is recovered as a special case of Eq.~(3.3) by setting $w_{k+1,m}^i = 1$ for all $i$, which reduces the procedure to simple majority voting. The uncertainty-aware variant generalizes this by weighting each client's answer according to its estimated uncertainty. This is particularly useful when an infrequent answer is nonetheless produced with consistently low uncertainty. In such cases, Eq.~(3.3) naturally upweights reliable yet rare signals while downweighting responses with high uncertainty. This highlights the role of uncertainty as an effective calibration mechanism within the aggregation process.
\end{tcolorbox}

\begin{tcolorbox}[
  enhanced,
  breakable,
  colback=lightblue!10, colframe=lightblue!90!black,
  title={Illustrative Example of UA-SCA},
  fonttitle=\bfseries
]
\textbf{Server query:} \emph{``Is tomato a fruit or a vegetable?''}

\textbf{Client submissions} (each client holds data from a different domain):

$\bullet$ \textbf{Client~1} (botany data, $u{=}0.08$): ``In botany, a tomato develops from the flower ovary and contains seeds.'' $\rightarrow$ Fruit.

$\bullet$ \textbf{Client~2} (culinary data, $u{=}0.82$): ``Tomatoes are commonly used in salads, sauces, and savory dishes rather than desserts.'' $\rightarrow$ Vegetable.

$\bullet$ \textbf{Client~3} (restaurant data, $u{=}0.88$): ``Tomatoes are categorized with vegetables on most restaurant menus and grocery labels.'' $\rightarrow$ Vegetable.

Without uncertainty, majority voting would select Vegetable (2 vs 1), which is contextually common but scientifically incorrect.

\textbf{Step 1 -- Grouping.} Responses are partitioned by final answer: Group~A (Fruit): Client~1; Group~B (Vegetable): Clients~2, 3.

\textbf{Step 2 -- Summarize.} Group~A: ``Botanically, tomatoes develop from flowers and contain seeds, which classifies them as fruits.'' Group~B: ``Tomatoes are widely treated as vegetables in cooking and food organization.''

\textbf{Step 3 -- SelfCritique.} Each client's trace is evaluated against the opposing group's summary and revised if contradicted:

$\bullet$ \textbf{Client~1}: Group~B discusses culinary usage, which does not contradict the botanical definition. \emph{No revision.} Revised trace: ``In botany, a tomato develops from the flower ovary and contains seeds.'' $\rightarrow$ Fruit.

$\bullet$ \textbf{Client~2}: Group~A provides a biological definition that conflicts with Client~2's assumption that culinary usage determines category. \emph{Partially revised.} Revised trace: ``Tomatoes are commonly treated as vegetables in cooking, although botanically they are fruits.'' $\rightarrow$ Fruit.

$\bullet$ \textbf{Client~3}: Group~A's evidence prompts reconsideration; grocery labeling reflects usage rather than scientific classification. \emph{Revised.} Revised trace: ``Tomatoes are often grouped with vegetables in restaurants and stores, but botanically they are fruits.'' $\rightarrow$ Fruit.

After self-critique, all three clients answer Fruit. However, this alone does not guarantee the correct outcome in general; the uncertainty weights provide a further safeguard.

\textbf{Step 4 -- Aggregate.} The server LLM receives the three revised traces with uncertainty weights $w_1{=}0.56$, $w_2{=}0.24$, $w_3{=}0.20$ (Eq.~\ref{eq:weight}; Client~1's low uncertainty yields the highest weight). It synthesizes a single reasoning chain: ``Although tomatoes are commonly treated as vegetables in culinary settings, botanically they develop from flowers and contain seeds, which classifies them as fruits.'' 

\textbf{Output:} Fruit.
\end{tcolorbox}
\label{app:ua_sca_example}

\subsubsection{Uncertainty-Aware Self-Critique Aggregation (UA-SCA)}
UA-SCA addresses disagreement among heterogeneous client reasoning paths by performing aggregation directly in the space of multi-step reasoning. For a given query, each client generates a reasoning trace and final answer together with an uncertainty score derived from its token-level predictive distribution, which reflects query-dependent reliability. UA-SCA first groups client responses by their final answers and summarizes the dominant reasoning pattern within each group using a server-side LLM. Each client’s reasoning–answer pair is then refined through self-critique by comparing it against summaries from alternative answer groups, allowing logical inconsistencies to be exposed and complementary intermediate steps to be incorporated. Finally, the revised reasoning paths are aggregated using uncertainty-aware weights that downweight unreliable or domain-mismatched predictions while amplifying confident and well-supported ones. By combining structured self-critique with uncertainty-aware weighting, UA-SCA resolves conflicts among heterogeneous reasoning paths without relying on majority agreement, while preserving informative intermediate logic throughout iterative federated reasoning rounds; the complete workflow of UA-SCA is described in Algorithm~\ref{alg:UA-SCA}.

We provide a step-by-step walkthrough below to illustrate how data heterogeneity causes conflicting reasoning, how self-critique partially resolves it, and how uncertainty weighting handles residual disagreements.

\section{Proof in Section \ref{sec:3}}\label{app:theory}

\subsection{Preliminaries for the Theoretical Analysis}\label{app: pre}

In this section, we lay out the basic formulation of in-context learning (ICL) for function classes, building on \cite{garg2022can,zhang2023trained}.

In ICL, the model processes and exploits sequential input called a prompt. A prompt consists of a sequence of input–output pairs $(x_1, y_1, \ldots, x_N, y_N, x_{\text{query}})$, where each $y_i$ corresponds to the evaluation of an unknown target function $h$ at $x_i$. Given such a sequence, the model’s task is to infer useful information from the examples and generate a prediction $\hat{y}(x_{\text{query}})$ for the query point $x_{\text{query}}$, aiming for $\hat{y}(x_{\text{query}}) \approx h(x_{\text{query}})$.

We now present a formal definition for models that learn from in-context examples, following \cite{zhang2023trained}.

\begin{definition}[(Definition 3.1 in \cite{zhang2023trained}) Trained on in-context examples]
\label{def:trained.on.incontext}
Let $\calDx$ be a distribution over an input space $\calX$, $\calH\subset \calY^ \calX$ a set of functions $\calX\to \calY$, and $\calDH$ a distribution over functions in $\calH$.  Let $\ell:\calY\times \calY \to \R$ be a loss function.   Let $\seqspace = \cup_{n\in \N} \{ (x_1, y_1, \dots, x_n, y_n): x_i \in \calX, y_i\in \calY \}$ be the set of finite-length sequences of $(x, y)$ pairs and let 
\[ \calF_\Theta = \{f_\theta : \seqspace\times \calX \to \calY,\, \theta\in \Theta\}\]
be a class of functions parameterized by $\theta$ in some set $\Theta$.  For $T>0$, we say that a model $f:\seqspace\times \calX \to \calY$ is \emph{trained on in-context examples of functions in $\calH$ under loss $\ell$ w.r.t. $(\calDH,\calD_x)$} if $f = f_{\theta^*}$ where $\theta^*\in \Theta$ satisfies
\begin{equation} \label{eq:icl.stochastic.opt.def}
\theta^* \in \mathrm{argmin}_{\theta \in \Theta} \E_{\prompt = (x_1, h(x_1), \dots, x_T, h(x_{T}), x_\query)} \left[ \ell\left(f_\theta(\prompt), h(x_\query) \right)\right],
\end{equation}
where $x_{i},x_\query \iid \calDx$ and $h\sim \calDH$ are independent.
We call $T$ the \emph{length of the prompts seen during training.}
\end{definition}

Let $E \in \R^{d_e \times d_T}$ denote an embedding matrix associated with a prompt 
$P = (x_1, y_1, \dots, x_T, y_T, x_{\query})$. 
Following \citet{zhang2023trained}, we form each column from $(x_i, y_i)^\top \in \R^{d+1}$ for $i=1,\dots,T$, and use $(x_{\query}, 0)^\top$ as the final column. 
For $x_i \in \R^d$ and $y_i \in \R$, we have $d_e = d+1$ and $d_T = T+1$.  Under the prompt representation $P$, the embedding matrix can be written as
\begin{equation}
E=E(P) \;=\;
\begin{pmatrix}
    x_1 & x_2 & \cdots & x_T & x_{\query} \\
    y_1 & y_2 & \cdots & y_T & 0
\end{pmatrix}
\;\in\; \R^{(d+1)\times (T+1)}.
\label{eq:embedding.matrix.prompt}
\end{equation}

We introduce weight matrices $\WQ,\WK \in \R^{d_k \times d_e}$ for the query and key, 
$\WV \in \R^{d_v \times d_e}$ for the value, $\WP \in \R^{d_e \times d_v}$ for projection, 
and a normalization constant $\rho > 0$.  

For the sake of tractability in theoretical analysis, following \citet{zhang2023trained}, we consider a single-layer \emph{linear self-attention (LSA)} model. 
This variant streamlines the standard self-attention by eliminating the softmax step and merging the projection matrices. 
Specifically, we define merged operators $\WKQ \in \R^{d_e \times d_e}$ (query–key) and $\WPV \in \R^{d_e \times d_e}$ (projection–value). 
Letting $\params = (\WKQ, \WPV)$, the LSA model can be formulated as
\begin{equation} \label{eq:lsa}
    f_{\mathrm{LSA}}(E;\params) \;=\; 
    E \;+\; \WPV E \cdot \frac{E^\top \WKQ E}{\rho}.
\end{equation}

The prediction corresponding to the query token $x_{\query}$ is given by the 
bottom-right entry of $f_{\lsa}$:
\[
    \widehat{y}_{\query} \;=\; \widehat{y}_{\query}(E;\params) 
    \;=\; [f_{\lsa}(E; \params)]_{(d+1), (T+1)}.
\]

To streamline the theoretical development, we restrict attention to in-context 
learning with linear predictors, in line with the setup of \citet{zhang2023trained}. 
We assume that all clients share the same sampling procedure for generating training prompts. 
Let $\Lambda$ be a positive definite covariance matrix. For each task 
indexed by $\tau \in \N$, we construct a training prompt
\[
    P_\tau = (x_{\tau,1}, h_\tau(x_{\tau,1}), \dots, x_{\tau,T}, h_\tau(x_{\tau,T}), x_{\tau,\query}).
\]
The task-specific parameter $w_\tau$ is drawn independently from $\mathcal{N}(0, I_d)$, 
the inputs $x_{\tau,i}$ and $x_{\tau,\query}$ are sampled i.i.d.\ from $\mathcal{N}(0, \Lambda)$, 
and the labels are defined by the linear rule $h_\tau(x) = \langle w_\tau, x \rangle$.  

Each prompt $P_\tau$ is then converted into an embedding matrix $E_\tau$ using 
the transformation in Eq.~\eqref{eq:embedding.matrix.prompt}:
\begin{equation*}
    E_\tau \;=\;
    \begin{pmatrix}
        x_{\tau,1} & x_{\tau,2} & \cdots & x_{\tau,T} & x_{\tau,\query} \\
        \langle w_\tau, x_{\tau,1}\rangle & \langle w_\tau, x_{\tau,2}\rangle & \cdots 
        & \langle w_\tau, x_{\tau,T}\rangle & 0
    \end{pmatrix}
    \in \R^{(d+1)\times (T+1)}.
\end{equation*}

The empirical risk evaluated over $B$ independent prompts is given by
\begin{equation}\label{eqn:emp_loss_maintext}
    \widehat L(\params) \;=\; \frac{1}{2B} \sum_{\tau=1}^B 
    \Big(\widehat y_{\tau,\query} - \langle w_\tau, x_{\tau,\query}\rangle \Big)^2.
\end{equation}

To study the limiting behavior, we introduce the population loss obtained as 
$B \to \infty$:
\begin{equation}\label{eqn:population_loss}
    L(\params) \;=\; \lim_{B\to \infty} \widehat L(\params) 
    \;=\; \frac{1}{2}\,
    \E_{w_\tau,\, x_{\tau,1},\dots, x_{\tau,T},\, x_{\tau,\query}}
    \!\left[ \big(\widehat y_{\tau,\query} - \langle w_\tau, x_{\tau,\query}\rangle\big)^2 \right].
\end{equation}
Here the expectation is taken over the random task weight $w_\tau \sim \mathcal{N}(0,I_d)$ 
and the covariates $\{x_{\tau,i}\}_{i=1}^T \cup \{x_{\tau,\query}\}$, 
drawn i.i.d.\ from $\mathcal{N}(0,\Lambda)$.  

We analyze optimization via the \emph{gradient flow} framework, which describes 
the continuous-time limit of gradient descent with infinitesimal step size. 
The dynamics of the parameters follow the ODE
\begin{equation}\label{eqn:gf}
    \frac{\mathrm{d}}{\mathrm{d}t}\, \params \;=\; - \nabla L(\params).
\end{equation}

In the remainder, we study gradient flow trajectories under initializations 
that satisfy the following assumptions in \cite{zhang2023trained}:

\begin{assumption}[Initialization (\cite{zhang2023trained})]\label{assume_init}
    Let $\sigma>0$ be a parameter, and let $\Theta \in \R^{d\times d}$ be any matrix satisfying $\snorm{\Theta \Theta^\top}_F =1$ and $\Theta \Lambda \neq 0_{d\times d}$. We assume
\begin{equation} \label{eq:wpv.wkq.init}
    \WPV(0)= 
    \sigma\begin{pmatrix}
    0_{d\times d} & 0_d \\
    0_d^\top & 1
    \end{pmatrix},\quad 
    \WKQ(0) = 
    \sigma \begin{pmatrix}
    \Theta \Theta^\top & 0_d \\ 
    0_d^\top & 0 
    \end{pmatrix}.
\end{equation}
\end{assumption}

under suitable initialization, gradient flow will converge to a global optimum.

\begin{theorem}[(Theorem 4.1 of \cite{zhang2023trained}) Convergence and limits]
   \label{thm:mainresult}
Consider the gradient flow of the linear self-attention network $f_\lsa$ defined in~\eqref{eq:lsa} over the population loss~\eqref{eqn:population_loss}.  Suppose the initialization satisfies Assumption~\ref{assume_init} with initialization scale $\sigma>0$ satisfying $\sigma^2\snorm{\Gamma}_{op}\sqrt d <2$ 
where we have defined
\[ \Gamma := \left(1 + \frac{1}{T}\right)\Lambda + \frac{1}{T}\operatorname{tr}(\Lambda) I_d \in \mathbb{R}^{d \times d}.\]
Then, the gradient flow converges to a global minimum of the population loss \eqref{eqn:population_loss}. Moreover, $\WPV$ and $\WKQ$ converge to $\WPV_*$ and $\WKQ_*$ respectively, where
\begin{equation}\label{eq:limit_expression}
\begin{aligned}
    \WKQ_* &=
    \left[\operatorname{\tr}\left(\Gamma^{-2}\right)\right]^{-\frac{1}{4}}
    \begin{pmatrix}
        \Gamma^{-1} & 0_{d} \\
        0_d^\top & 0
    \end{pmatrix}, \qquad
   \WPV_* =
    \left[\operatorname{\tr}\left(\Gamma^{-2}\right)\right]^{\frac{1}{4}}
    \begin{pmatrix}
        0_{d \times d} & 0_{d} \\
        0_d^\top & 1
    \end{pmatrix}.
\end{aligned}
\end{equation} 
\end{theorem}

At the global optimum with parameters $\WKQ_\ast$ and $\WPV_\ast$, 
the prediction for the query token takes the form
\begin{align}
    \widehat y_{\query} 
    &= 
    \begin{pmatrix}
        0_d^\top & 1
    \end{pmatrix}
    \begin{pmatrix}
        \frac{1}{M}\sum_{i=1}^{M} \testx_i \testx_i^\top 
        + \frac{1}{M}\testx_{\query} \testx_{\query}^\top 
        & \frac{1}{M}\sum_{i=1}^{M} \testx_i \testy_i \\
        \frac{1}{M}\sum_{i=1}^{M} \testx_i^\top \testy_i 
        & \frac{1}{M}\sum_{i=1}^{M} \testy_i^2
    \end{pmatrix}
    \begin{pmatrix}
        \Gamma^{-1} & 0_d \\
        0_d^\top & 0
    \end{pmatrix}
    \begin{pmatrix}
        \testx_\query \\
        0
    \end{pmatrix} \notag \\
    &= \testx_\query^\top \Gamma^{-1} 
    \left(\frac{1}{M} \sum_{i=1}^{M} \testy_i \testx_i \right).
    \label{eq:prediction.trained.transformer}
\end{align}

\subsection{Proof of Theorem \ref{thm:1}}
\label{app:thm1_proof}

According to the above preliminary theoretical results adopted from \cite{zhang2023trained}, we have the following analysis.
\begin{proof}

First by our assumption of the client and server data, we have 
\begin{align}
    q_n^i \sim N(0, \Lambda),\ q_m \sim N(0, \Lambda),\notag
\end{align}
and 
\begin{align}
    a_n^i = (q_n^i)^\top \theta + \epsilon_n^i, \epsilon_n^i \sim N(0, \sigma_i^2), \|\theta\| \leq 1,\notag
\end{align}
where the variance $\sigma_i^2$ are different from client to client, representing the heterogeneity over clients. Then we have the following inequality holds with probability at least $1-\delta$: 
\begin{align}
&\bigg\|I - \Gamma^{-1}\frac{\sum_{m=1}^M q_m(q_m)^\top}{M}\bigg\|_{\text{op}}\notag \\
& \leq \bigg\|\Gamma^{-1}\Lambda \bigg(I - \Lambda^{-1}\frac{\sum_{m=1}^M q_m(q_m)^\top}{M}\bigg)\bigg\|_{\text{op}} + \bigg\|I - \Gamma^{-1}\Lambda \bigg\|_{\text{op}}\notag \\
& \leq \|\Gamma^{-1}\Lambda\|\cdot 4\cdot \bigg( \sqrt{\frac{d \log \delta^{-1}}{M}} + \frac{d \log \delta^{-1}}{M}\bigg) + \bigg\|I - \Gamma^{-1}\Lambda \bigg\|_{\text{op}}\notag \\
& \leq 10 \cdot \sqrt{\frac{d \log \delta^{-1}}{M}}, \label{help:11}
\end{align}
where for the second inequality, we use the standard matrix concentration inequality, for the last one, we use the fact that $T$ is large enough and $M \geq 4d \log \delta^{-1}$, where
\begin{align}
    \bigg\|I - \Gamma^{-1}\Lambda \bigg\|_{\text{op}} = \frac{\lambda_{\min} (\Lambda) + \tr(\Lambda)}{ T\lambda_{\min} (\Lambda) + \tr(\Lambda)} \leq \sqrt{\frac{d \log \delta^{-1}}{M}} \leq \frac{1}{4}.\notag
\end{align}
Similarily, with probability at least $1-\delta$, we have for all $i \in [L]$, 
\begin{align}
    \bigg \|I - \Gamma^{-1}\frac{\sum_{n=1}^N q_n^i (q_n^i)^\top}{N}\bigg \|_{\text{op}} \leq 10 \cdot \sqrt{\frac{d \log (L\delta^{-1})}{N}}.\label{help:22}
\end{align}

First, with Eq.(\ref{eq:prediction.trained.transformer}) and the algorithm design of Algorithm \ref{alg:FERA}, we have the following guarantees
\begin{align}
\hat{a}^i_{k,n} &= (q_n^i)^\top \Gamma^{-1}\left(\frac{1}{M}\sum_{m=1}^M q_m a_{k,m} \right) \label{eq: incontext 1}\,,\\
            a_{k+1, m}^i &= q_m^\top \Gamma^{-1}\left(\frac{1}{2N}\sum_{n=1}^N q_n^i (a_{n}^i + \hat a_{k,n}^i) \right) \label{eq: incontext 2}\,.
\end{align}

Then we can prove the theorem by induction.

Assume $a_{k,m}=\theta_k^\top q_m$ holds for episode $k$, which trivially holds at episode 1 with $\theta_1=0$ as $a_{1,m}=0, \forall m\in[M]$. According to Eq.(\ref{eq: incontext 1}), Eq.(\ref{eq: incontext 2}), and $ a_{k+1,m} = \sum_{i=1}^L w_{k+1, m}^i a_{k+1, m}^i $, we have
\begin{align}
        a^i_{k+1,m}&=q_m^\top \Gamma^{-1}\left(\frac{1}{2N}\sum_{n=1}^N q_n^i (a_{n}^i + \hat a_{k,n}^i) \right) \notag\\
        &=q_m^\top \Gamma^{-1}\left(\frac{1}{2N}\left(\sum_{n=1}^N q_n^i a_{n}^i + \sum_{n=1}^N q_n^i(q_n^i)^\top \Gamma^{-1}\left(\frac{1}{M}\sum_{m=1}^M q_m a_{k,m} \right)\right) \right)\notag\\
        &=q_m^\top \Gamma^{-1}\left(\frac{1}{2N}\left(\sum_{n=1}^N q_n^i a_{n}^i + \sum_{n=1}^N q_n^i(q_n^i)^\top \Gamma^{-1}\left(\frac{1}{M}\sum_{m=1}^M q_m q_m^\top \right)\theta_k\right) \right)\notag\,,
\end{align}
Then, we have
\begin{align}
    a_{k+1,m}&=\sum_{i=1}^L w^i_{k+1,m}a^i_{k+1,m}\notag\\
    &=\sum_{i=1}^L w^i_{k+1,m} q_m^\top \Gamma^{-1}\left(\frac{1}{2N}\left(\sum_{n=1}^N q_n^i a_{n}^i + \sum_{n=1}^N q_n^i(q_n^i)^\top \Gamma^{-1}\left(\frac{1}{M}\sum_{m=1}^M q_m q_m^\top \right)\theta_k\right) \right)\notag\\
    &=q_m^\top \left(\Gamma^{-1}\frac{\sum_{i=1}^L w^i_{k+1,m}\sum_{n=1}^N q_n^i a_n^i}{2N}+\frac{\sum_{i=1}^Lw^i_{k+1,m}\Gamma^{-1}\sum_{n=1}^N q_n^i(q_n^i)^\top \Gamma^{-1}\sum_{m=1}^M q_mq_m^\top }{2NM}\cdot \theta_k\right)\notag\\
    & = q_m^\top \Bigg(\frac{1}{2}\underbrace{\sum_{i=1}^Lw^i_{k+1,m}\bigg(\Gamma^{-1}\frac{\sum_{n=1}^N q_n^i(q_n^i)^\top}{N}\bigg)  \bigg(\Gamma^{-1}\frac{\sum_{m=1}^M q_m(q_m)^\top}{M}\bigg)}_{H_k} \cdot \theta_k\notag\\
    &+\frac{1}{2}\cdot \underbrace{\sum_{i=1}^L w^i_{k+1,m}\Gamma^{-1}\frac{\sum_{n=1}^N q_n^i a_n^i}{N}}_{\bar \theta_k}\Bigg).
\end{align}
Then appreately, we have shown that for each $k$, $a_{k,m} = \theta_k^\top q_m$ can be written as a linear function of the query $q_m$. Next we bound $\bar \theta_k$ and $H_k$ separately. We have 
\begin{align}
    \bar \theta_k &= \sum_{i=1}^L w^i_{k+1,m}\Gamma^{-1}\frac{\sum_{n=1}^N q_n^i (q_n^i)^\top \theta + q_n^i\epsilon_n^i}{N}\notag \\
    & = \theta - \underbrace{\sum_{i=1}^L w^i_{k+1,m}\bigg(I - \Gamma^{-1}\frac{\sum_{n=1}^N q_n^i (q_n^i)^\top}{N}\bigg)\theta}_{A_{k,m}} + \underbrace{\sum_{i=1}^L w^i_{k+1,m}\Gamma^{-1}\frac{ \sum_{n=1}^N q_n^i\epsilon_n^i}{N}}_{B_{k,m}},
\end{align}

and
\begin{align}
    H_k &= \sum_{i=1}^Lw^i_{k+1,m}\bigg(I - \bigg(I - \Gamma^{-1}\frac{\sum_{n=1}^N q_n^i(q_n^i)^\top}{N}\bigg)\bigg)  \bigg(I - \bigg(I - \Gamma^{-1}\frac{\sum_{m=1}^M q_m(q_m)^\top}{M}\bigg)\bigg)\notag \\
    & = I - \underbrace{\sum_{i=1}^Lw^i_{k+1,m}\bigg(I - \Gamma^{-1}\frac{\sum_{n=1}^N q_n^i(q_n^i)^\top}{N}\bigg) - \sum_{i=1}^Lw^i_{k+1,m}\bigg(I - \Gamma^{-1}\frac{\sum_{m=1}^M q_m(q_m)^\top}{M}\bigg)}_{C_{k,m}}\notag \\
    &\quad + \underbrace{\sum_{i=1}^Lw^i_{k+1,m}\bigg(I - \Gamma^{-1}\frac{\sum_{n=1}^N q_n^i(q_n^i)^\top}{N}\bigg)\bigg(I -  \Gamma^{-1}\frac{\sum_{m=1}^M q_m(q_m)^\top}{M}\bigg)}_{D_{k,m}}.
\end{align}

\paragraph{Bound $A_{k,m}$. }
For $A_{k,m}$ by \eqref{help:22}, we have 
\begin{align}
    \|A_{k,m}\|_{\text{op}} \leq \|\theta\|\bigg \|I - \Gamma^{-1}\frac{\sum_{n=1}^N q_n^i (q_n^i)^\top}{N}\bigg \|_{\text{op}} \leq 10 \cdot \sqrt{\frac{d \log (L\delta^{-1})}{N}}.\label{help:881}
\end{align}

\paragraph{Bound $B_{k,m}$. }
For $B_{k,m}$, recall that $q_n^i \sim N(0, \Lambda)$ and $\epsilon_n^i \sim N(0,\sigma_i^2)$, then $\Lambda^{-1/2} q_n^i \sim N(0, I)$. Then applying concentration inequality on sub-exponential random vectors, with probability at least $1-\delta$, we have
\begin{align}
   B_{k,m}&: = \Gamma^{-1}\Lambda^{1/2} \sum_{i=1}^L w^i_{k+1,m}\frac{ \sum_{n=1}^N \Lambda^{-1/2} q_n^i\epsilon_n^i}{N}\notag \\
   & \leq \|\Gamma^{-1}\Lambda^{1/2}\| \sum_{i=1}^L w^i_{m}\cdot\bigg\|\frac{ \sum_{n=1}^N \Lambda^{-1/2} q_n^i\epsilon_n^i}{N}\bigg\|\notag \\
   & \leq \lambda_{\min}^{-1/2}(\Lambda)\cdot  \sum_{i=1}^L w^i_{m} \cdot 10 \sigma_i\bigg( \sqrt{\frac{d \log(L\delta^{-1})}{N}} + \frac{d^{1.5} \log(L\delta^{-1})}{N}\bigg). \label{help:882}
\end{align}

\paragraph{Bound $C_{k,m}$. }
For $C_{k,m}$, using \eqref{help:11} and \eqref{help:22}, we have 
\begin{align}
    \|C_{k,m}\|_{\text{op}} \leq \bigg(\bigg \|I - \Gamma^{-1}\frac{\sum_{n=1}^N q_n^i(q_n^i)^\top}{N}\bigg\|_{\text{op}}  + \bigg\|I - \Gamma^{-1}\frac{\sum_{m=1}^M q_m(q_m)^\top}{M}\bigg\|_{\text{op}} \bigg) \leq 20\cdot \sqrt{\frac{d \log (L\delta^{-1})}{\min\{M,N\}}}. \label{help:883}
\end{align}

\paragraph{Bound $D_{k,m}$. }
For $D_{k,m}$, using \eqref{help:11} and \eqref{help:22}, we have 
\begin{align}
    \|D_{k,m}\|_{\text{op}} \leq \bigg \|I - \Gamma^{-1}\frac{\sum_{n=1}^N q_n^i(q_n^i)^\top}{N}\bigg\|_{\text{op}}  \cdot \bigg\|I - \Gamma^{-1}\frac{\sum_{m=1}^M q_m(q_m)^\top}{M}\bigg\|_{\text{op}} \leq 100 \cdot \frac{d \log (L\delta^{-1})}{\min\{M, N\}}. \label{help:884}
\end{align}

Therefore, combining \eqref{help:881} to \eqref{help:884}, we have 
\begin{align}
    \|\theta_{k+1}  - \theta\| &= \bigg\|\bigg(\frac{1}{2} + \frac{C_{k,m} - D_{k,m}}{2}\bigg)(\theta_k - \theta)  -\frac{A_{k,m} - B_{k,m}}{2}  - \frac{C_{k,m} - D_{k,m}}{2}\theta\bigg\|\notag \\
    & \leq \bigg\|\frac{1}{2} + \frac{C_{k,m} - D_{k,m}}{2}\bigg\|\|\theta_k - \theta\| + \bigg\|\frac{C_{k,m} - D_{k,m}}{2}\bigg\| + \bigg\|\frac{A_{k,m} - B_{k,m}}{2}\bigg\|.\label{help:end}
\end{align}
Since 
\begin{align}
    \bigg\|\frac{C_{k,m} - D_{k,m}}{2}\bigg\| < 10\cdot \sqrt{\frac{d \log (L\delta^{-1})}{\min\{M,N\}}} +  50 \cdot \frac{d \log (L\delta^{-1})}{\min\{M, N\}}<\frac{1}{4}
\end{align}
and 
\begin{align}
    \bigg\|\frac{A_{k,m} - B_{k,m}}{2}\bigg\| \leq 5 \cdot \sqrt{\frac{d \log (L\delta^{-1})}{N}} +5 \lambda_{\min}^{-1/2}(\Lambda)\cdot  \sum_{i=1}^L w^i_{m}  \sigma_i\bigg( \sqrt{\frac{d \log(L\delta^{-1})}{N}} + \frac{d^{1.5} \log(L\delta^{-1})}{N}\bigg)
\end{align}
Then applying recursion onto \eqref{help:end}, we have that for all $k$, 
\begin{align}
    \|\theta_k - \theta\| &\leq O\bigg(\sqrt{\frac{d \log (L\delta^{-1})}{\min\{M,N\}}} +  \sqrt{\frac{d \log(L\delta^{-1})}{N}} +  \sqrt{\frac{d \log(L\delta^{-1})}{N}}\lambda_{\min}^{-1/2}(\Lambda) \sum_{i=1}^L w^i_{m}  \sigma_i \bigg).\notag
\end{align}
\end{proof}

\section{Prompt}
\label{app:prompt}
\label{app:alg}
\subsection{Server Query Dataset Initialize}
\label{app:server_init} 
To initialize the query dataset, the server distributes queries to multiple clients. Each client employs its local LLM to generate the responses for the assigned queries. The responses are then returned to the server, which aggregates them to form the initial server-side query dataset. The prompts used to guide this response generation are described as follows.

\begin{tcolorbox}[
  enhanced,
  breakable,
  colback=lightblue!10, colframe=lightblue!90!black,
  title={Server Query Initialization Prompt for MMLU-Pro \& AQUA-RAT Reasoning Benchmark},
  fonttitle=\bfseries
]
\small\ttfamily
\hypertarget{server-query-prompt}{}
You are a knowledgeable assistant. For the following multiple-choice question, briefly explain your reasoning (no more than \texttt{\{sentences\_limit\}} sentences), then end with the exact sentence: \texttt{The answer is (X).}

\medskip
\textbf{Rules:}
\begin{enumerate}[leftmargin=*, itemsep=1pt, topsep=2pt]
  \item \texttt{X} must be the option letter only (A/B/C/D/\dots). Do not include the option text.
  \item Do not include any content after the final sentence.
  \item Keep your entire response within \texttt{\{token\_limit\}} tokens.
\end{enumerate}

\medskip
\textbf{Question:} \texttt{\{query\}}\\
\textbf{Answer:} Let’s think step by step. \texttt{\{text\}}
\end{tcolorbox}

\begin{tcolorbox}[
  enhanced,
  breakable,
  colback=lightblue!10, colframe=lightblue!90!black,
  title={Server Query Initialization Prompt for MMLU-Pro \& AQUA-RAT Standard QA Benchmark},
  fonttitle=\bfseries
]
\small\ttfamily
\hypertarget{server-query-prompt}{}
You are taking a multiple-choice question. Read the following question carefully and select the single best answer. Do \emph{not} explain your reasoning. Output only the final answer choice letter (A, B, C, D, \dots).

\medskip
\textbf{Rules:}
\begin{enumerate}[leftmargin=*, itemsep=1pt, topsep=2pt]
  \item The output must be a single uppercase letter (A/B/C/D/\dots) with no punctuation or extra text.
  \item Do not include any explanation or content after the answer.
  \item Keep the response within \texttt{\{token\_limit\}} tokens.
\end{enumerate}

\medskip
\textbf{Question:} \texttt{\{query\}}

\medskip
\textbf{Answer:}
\end{tcolorbox}

\begin{tcolorbox}[
  enhanced,
  breakable,
  colback=lightblue!10, colframe=lightblue!90!black,
  title={Server Query Initialization Prompt for GSM8K Reasoning Benchmark},
  fonttitle=\bfseries
]
\small\ttfamily
\hypertarget{server-query-prompt}{}
You are a knowledgeable assistant. For the following math question, briefly explain your reasoning (no more than \texttt{\{sentences\_limit\}} sentences), then end with the exact sentence: \texttt{The answer is X.}

\medskip
\textbf{Rules:}
\begin{enumerate}[leftmargin=*, itemsep=1pt, topsep=2pt]
  \item \texttt{X} must be a single numeric value (e.g., \texttt{12}, \texttt{-3/5}, \texttt{7.25}); no units or extra text.
  \item If \texttt{X} is a fraction, reduce it to simplest terms; if a decimal, use standard form without trailing zeros.
  \item Do not include any content after the final sentence.
  \item Keep the entire response within \texttt{\{token\_limit\}} tokens.
\end{enumerate}

\medskip
\textbf{Question:} \texttt{\{query\}}

\medskip
\textbf{Answer: Let’s think step by step.}
\end{tcolorbox}

\subsection{Client Response Generation for Server Queries}
\label{app:client_resp}
The framework of FERA is outlined in Algorithm~\ref{alg:FERA}. In Step~2, each client relabels its local data by constructing prompts that incorporate demonstrations selected from the server dataset. In Step~3, the client relabels the server data by constructing prompts that include demonstrations drawn from its updated local dataset. The prompts used in Step~2 and Step~3 are described below. Unless otherwise specified, the default number of demonstrations is set to 5.

\begin{tcolorbox}[
  enhanced,
  breakable,
  colback=lightblue!10, colframe=lightblue!90!black,
  title={Client Prediction Prompt for MMLU-Pro \& AQUA-RAT Reasoning Benchmark},
  fonttitle=\bfseries
]
\small\ttfamily
\hypertarget{client-prediction-prompt-mmlu-aquarat}{}
You are tasked with answering multiple-choice math questions. Below are several example questions with their step-by-step reasoning and final answers. After reviewing these examples, you will be presented with a new question to answer.

\medskip
\textbf{Guidelines:}
\begin{enumerate}[leftmargin=*, itemsep=1pt, topsep=2pt]
  \item Provide clear, concise, and logically coherent step-by-step reasoning (at most \texttt{\{sentences\_limit\}} sentences).
  \item End with the exact sentence: \texttt{The answer is (X).}
  \item \texttt{X} must be the option letter only (A, B, C, D, \dots); do not include the option text.
  \item Include no additional content after the final answer sentence.
  \item Keep the complete response within \texttt{\{token\_limit\}} tokens.
\end{enumerate}

\medskip
\textbf{Examples:}\\
\texttt{\{examples\}}

\medskip
\textbf{Question:}\\
\texttt{\{query\}}

\medskip
\textbf{Answer:} Let’s think step by step.
\end{tcolorbox}

\begin{tcolorbox}[
  enhanced,
  breakable,
  colback=lightblue!10, colframe=lightblue!90!black,
  title={Client Prediction Prompt for MMLU-Pro \& AQUA-RAT Standard QA Benchmark},
  fonttitle=\bfseries
]
\small\ttfamily
\hypertarget{client-prediction-prompt-mmlu-aquarat}{}
You are taking a multiple-choice question. Below are several example questions with their final answers. After reviewing these examples, you will be presented with a new question to answer.

\medskip
\textbf{Guidelines:}
\begin{enumerate}[leftmargin=*, itemsep=1pt, topsep=2pt]
  \item Read the question carefully and select the single best answer.
  \item Do \emph{not} explain your reasoning.
  \item Output only the final answer choice letter (A, B, C, D, \dots); do not include the option text.
  \item Do not include any additional content after the answer.
\end{enumerate}

\medskip
\textbf{Examples:}\\
\texttt{\{examples\}}

\medskip
\textbf{Question:}\\
\texttt{\{query\}}

\medskip
\textbf{Answer:}
\end{tcolorbox}

\begin{tcolorbox}[
  enhanced,
  breakable,
  colback=lightblue!10, colframe=lightblue!90!black,
  title={Client Prediction Prompt for GSM8K Reasoning Benchmark},
  fonttitle=\bfseries
]
\small\ttfamily
\hypertarget{client-prediction-prompt}{}
You are tasked with answering math questions in this domain. Below are several example questions with their step-by-step reasoning and final answers. After reviewing these examples, you will be presented with a new question to answer.

\medskip
\textbf{Guidelines:}
\begin{enumerate}[leftmargin=*, itemsep=1pt, topsep=2pt]
  \item Provide clear, concise, and logically coherent step-by-step reasoning.
  \item End your response with the exact sentence: \texttt{The answer is X.}
  \item \texttt{X} must be a single numeric value (e.g., \texttt{12}, \texttt{-3/5}, \texttt{7.25}); no units or extra text.
  \item If \texttt{X} is a fraction, reduce it to simplest terms; if a decimal, use standard form without trailing zeros.
  \item Include no additional content after the final answer sentence.
  \item Keep the complete response within \texttt{\{token\_limit\}} tokens.
\end{enumerate}

\medskip
\textbf{Examples:}\\
\texttt{\{examples\}}

\medskip
\textbf{Question:}\\
\texttt{\{query\}}

\medskip
\textbf{Answer:} Let’s think step by step.
\end{tcolorbox}
\subsection{Uncertainty-Aware Aggregration}
\label{app:ua_agg}
In the FERA framework, Uncertainty-Aware Aggregation is employed on the server side to combine responses from clients. The weights used in this aggregation are first computed according to Equation~\ref{eq:www}, after which uncertainty is leveraged to guide the aggregation process. For the standard QA task, we propose the UA-WA algorithm, described in Section~\ref{sec:aggregation}, while for complex reasoning tasks we design the UA-SCA algorithm, detailed in Algorithm~\ref{alg:UA-SCA}. The prompts utilized in UA-SCA are presented below.

\begin{tcolorbox}[
  enhanced,
  breakable,
  colback=lightblue!10, colframe=lightblue!90!black,
  title={Summarize},
  fonttitle=\bfseries
]
\small\ttfamily
\hypertarget{client-prediction-prompt}{}
You are a cognitive reasoning analyst tasked with examining \texttt{\{len(reasoning\_list)\}} reasoning responses derived from the following question.
\medskip

\textbf{Question:}\\
\texttt{\{question\}}
\medskip

\textbf{Reasoning Responses:}\\
\texttt{\{reasoning\_formatted\}}
\medskip

\textbf{Analysis Objective:}\\
Provide a concise analytical characterization of this reasoning cluster. Identify the cognitive patterns, methodological strategies, and structural similarities across the responses.
\medskip

\textbf{Characterization Guidelines:}
\begin{enumerate}[leftmargin=*, itemsep=1pt, topsep=2pt]
  \item Synthesize the distinguishing features of these reasoning responses
  \item Emphasize their reasoning methodology and structural organization
  \item Identify common problem-solving paradigms across responses
  \item Present your analysis within \texttt{\{token\_limit\}} tokens
  \item Capture the essential cognitive characteristics of this cluster
\end{enumerate}
\medskip

\textbf{Your Analysis:}
\end{tcolorbox}

\begin{tcolorbox}[
  enhanced,
  breakable,
  colback=lightblue!10, colframe=lightblue!90!black,
  title={SelfCritique for MMLU-Pro and AQUA-RAT Benchamrk},
  fonttitle=\bfseries
]
\small\ttfamily
\hypertarget{self-critique-prompt}{}
You are improving a reasoning response by incorporating insights from conflicting reasoning approaches.
\medskip

\textbf{Original Question:}\\
\texttt{\{question\}}
\medskip

\textbf{Target Reasoning Response (to be improved):}\\
\texttt{\{target\_response\}}
\medskip

\textbf{Alternative Approaches Summary:}\\
\texttt{\{alternatives\_formatted\}}
\medskip

\textbf{Task:}\\
Create an enhanced version of the target reasoning response by incorporating valuable insights from the alternative approaches. Identify reasoning elements, methodologies, or perspectives from the conflicting summaries that could strengthen the original response.
\medskip

\textbf{Requirements:}
\begin{enumerate}[leftmargin=*, itemsep=1pt, topsep=2pt]
  \item Use the target reasoning response as your foundation
  \item Extract valuable insights from the alternative approaches
  \item Integrate these insights to create a more comprehensive response
  \item Maintain logical consistency throughout
  \item Present only the final improved reasoning
  \item Conclude with: ``The answer is (X)'' where X is the option letter only (A/B/C/D/...)
  \item Exclude option text after the letter
  \item Limit response to \texttt{\{token\_limit\}} tokens
  \item Omit meta-commentary or explanatory analysis
\end{enumerate}
\medskip

\textbf{Improved Response:}
\end{tcolorbox}

\begin{tcolorbox}[
  enhanced,
  breakable,
  colback=lightblue!10, colframe=lightblue!90!black,
  title={SelfCritique for GSM8K Benchamrk},
  fonttitle=\bfseries
]
\small\ttfamily
\hypertarget{self-critique-prompt}{}
You are improving a reasoning response by incorporating insights from conflicting reasoning approaches.
\medskip

\textbf{Original Question:}\\
\texttt{\{question\}}
\medskip

\textbf{Target Reasoning Response (to be improved):}\\
\texttt{\{target\_response\}}
\medskip

\textbf{Alternative Approaches Summary:}\\
\texttt{\{alternatives\_formatted\}}
\medskip

\textbf{Task:}\\
Create an enhanced version of the target reasoning response by incorporating valuable insights from the alternative approaches. Identify reasoning elements, methodologies, or perspectives from the conflicting summaries that could strengthen the original response.
\medskip

\textbf{Requirements:}
\begin{enumerate}[leftmargin=*, itemsep=1pt, topsep=2pt]
  \item Start with the target reasoning as your foundation
  \item Identify useful insights from the conflicting summaries that could improve the reasoning
  \item Integrate these insights to create a stronger, more comprehensive reasoning
  \item Maintain logical consistency throughout
  \item Present only the final improved reasoning
  \item End with "The answer is X".
  \item Limit response to \texttt{\{token\_limit\}} tokens
  \item No meta-commentary or analysis explanation
\end{enumerate}
\medskip

\textbf{Improved Response:}
\end{tcolorbox}


\begin{tcolorbox}[
  enhanced,
  breakable,
  colback=lightblue!10, colframe=lightblue!90!black,
  title={Aggregation for MMLU-Pro and AQUA-RAT Benchmark},
  fonttitle=\bfseries
]
\small\ttfamily
\hypertarget{aggregation-prompt}{}
You are synthesizing multiple reasoning responses to create a single, unified reasoning path.
\medskip

\textbf{Context:}\\
You are given several reasoning responses to the same question, each from a different client. A final answer has been determined by majority vote. Each response includes a confidence score (higher = more confident).
\medskip

\textbf{Question:}\\
\texttt{\{question\}}
\medskip

\textbf{Client Responses (with confidence scores):}\\
\texttt{\{client\_entries\}}
\medskip

\textbf{Task:}\\
Produce a single, concise, professional, and logically coherent merged reasoning response that synthesizes the reasoning leading to the final answer.
\medskip

\textbf{Requirements:}
\begin{enumerate}[leftmargin=*, itemsep=1pt, topsep=2pt]
  \item Synthesize the reasoning leading to the final answer
  \item Give greater weight to reasoning from higher-confidence responses
  \item Avoid unnecessary repetition or irrelevant details
  \item Keep the ENTIRE response (reasoning + final answer) within \texttt{\{token\_limit\}} tokens
  \item End with: ``The answer is (X)'' where X is the option letter only (A/B/C/D/...)
  \item Do not include the option text after the letter
  \item Maintain professional and logical coherence throughout
\end{enumerate}
\medskip

\textbf{Merged Reasoning Response:}
\end{tcolorbox}

\begin{tcolorbox}[
  enhanced,
  breakable,
  colback=lightblue!10, colframe=lightblue!90!black,
  title={Aggregation for MMLU-Pro and AQUA-RAT Benchmark},
  fonttitle=\bfseries
]
\small\ttfamily
\hypertarget{aggregation-prompt}{}
You are synthesizing multiple reasoning responses to create a single, unified reasoning path.
\medskip

\textbf{Context:}\\
You are given several reasoning responses to the same question, each from a different client. Each response includes a confidence score (higher = more confident).
\medskip

\textbf{Question:}\\
\texttt{\{question\}}
\medskip

\textbf{Client Responses (with confidence scores):}\\
\texttt{\{client\_entries\}}
\medskip

\textbf{Task:}\\
Produce a single, concise, professional, and logically coherent merged reasoning response that synthesizes the reasoning leading to the final answer.
\medskip

\textbf{Requirements:}
\begin{enumerate}[leftmargin=*, itemsep=1pt, topsep=2pt]
  \item Synthesize the reasoning leading to the final answer
  \item Give greater weight to reasoning from higher-confidence responses
  \item Avoid unnecessary repetition or irrelevant details
  \item Keep the ENTIRE response (reasoning + final answer) within \texttt{\{token\_limit\}} tokens
  \item End with: ``The answer is (X)'' where X is the option letter only (A/B/C/D/...)
  \item Do not include the option text after the letter
  \item Maintain professional and logical coherence throughout
\end{enumerate}
\medskip

\textbf{Merged Reasoning Response:}
\end{tcolorbox}

\begin{tcolorbox}[
  enhanced,
  breakable,
  colback=lightblue!10, colframe=lightblue!90!black,
  title={Aggregation for GSM8K Benchmark},
  fonttitle=\bfseries
]
\small\ttfamily
\hypertarget{aggregation-prompt}{}
You are synthesizing multiple reasoning responses to create a single, unified solution path.
\medskip

\textbf{Context:}\\
You are given several reasoning responses to the same question, each from a different client. Each response includes a confidence score (higher = more confident).
\medskip

\textbf{Question:}\\
\texttt{\{question\}}
\medskip

\textbf{Client Responses (with confidence scores):}\\
\texttt{\{client\_entries\}}
\medskip

\textbf{Task:}\\
Produce a single, concise, professional, and logically coherent merged reasoning response that synthesizes the reasoning leading to the final answer.
\medskip

\textbf{Requirements:}
\begin{enumerate}[leftmargin=*, itemsep=1pt, topsep=2pt]
  \item Synthesize the reasoning leading to the final answer
  \item Give greater weight to reasoning from higher-confidence responses
  \item Avoid unnecessary repetition or irrelevant details
  \item Keep the ENTIRE response (reasoning + final answer) within \texttt{\{token\_limit\}} tokens
  \item End with the exact sentence: ``The answer is X.''
  \item X must be a single numeric value (e.g., 12, -3/5, 7.25)
  \item No units or extra text after the answer
\end{enumerate}
\medskip

\textbf{Merged Reasoning Response:}
\end{tcolorbox}

\subsection{FERA Variants}
\label{app:FERA_variants}
\paragraph{FERA-GT.} To evaluate the effectiveness of FERA, we compare it with a simplified baseline, FERA-GT. In this baseline, the server issues a query to clients, who independently retrieve the top-$C$ question–answer pairs using the MMR demonstration strategy. These retrieved pairs serve as fixed context for generating client responses, which are then aggregated by the server to produce the final answer. Importantly, FERA-GT completes this process in a single communication round without iterative context refinement. In contrast, FERA employs multiple rounds of interaction, enabling dynamic context updates and progressively enhanced answer quality—underscoring its adaptability and superior reasoning depth.

\paragraph{FERA-Q.} FERA-Q corresponds to a simplified setting where the LLM is limited to predicting only the final answer to each question, without generating intermediate reasoning steps. In this setup, when clients select demonstrations from their local datasets, they discard the reasoning trajectories \( S_{\{1:T\}} \) and retain only the final answers \( a \). This setting reflects scenarios in which the LLM lacks the ability to handle or benefit from step-by-step reasoning. During inference, each client uses the selected final-answer-only demonstrations to generate predictions for the server-issued query set. The predicted answers are then sent back to the server. The server performs aggregation using the \textit{Uncertainty-Aware Weighted Averaging} (UA-WA) strategy, which weights client responses based on their confidence scores. The aggregated results are used to update the global query--answer set, enabling iterative refinement even without explicit reasoning steps.

\paragraph{FERA-Free.} FERA-Free refers to a setting where clients have local LLMs but no labeled data—only question prompts are available locally. Since clients cannot construct demonstrations from their own data, they rely entirely on query-answer pairs provided by the server. In each round, the server distributes example queries and answers. Clients use these examples to prompt their local LLMs and generate responses to new server queries. The server then aggregates responses using uncertainty-aware methods (UA-WA or UA-SCA) and refines the query-answer set for the next round. This setup captures a practical constraint where local data may be unlabeled due to privacy, cost, or domain limitations. FERA-Free shows that meaningful collaboration is possible even without local supervision by leveraging server-side guidance.

\section{Experiment Setup}
\label{app:exp_setup}
\begin{figure}
    \centering
    \includegraphics[width=\linewidth]{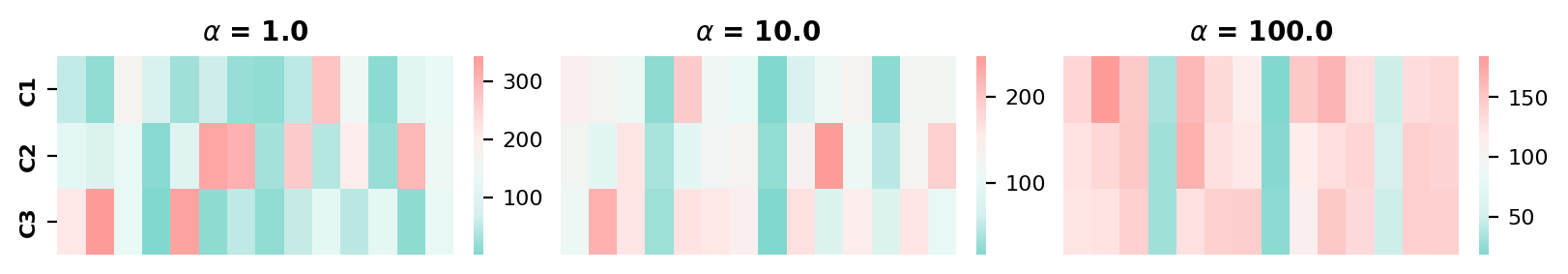}
         \captionof{figure}{Client dataset distribution under different $\alpha$ settings on the MMLU-Pro benchmark.}
        \label{fig:data_distribution_heatmap}
\end{figure}

\subsection{Construction of Client Datasets}
\label{app:client_data}

To evaluate FERA under realistic heterogeneous and domain-specialized conditions, we simulate a federated environment by partitioning each benchmark into client-specific subsets.

\textbf{MMLU-Pro.} 
MMLU-Pro spans 14 diverse subject areas—including physics, medicine, philosophy, economics, and law—which naturally provides a multi-domain foundation for studying specialized client expertise. To introduce controlled heterogeneity across clients, we construct client partitions using a Dirichlet-based sampling scheme. Specifically, for each client, we draw a label-distribution vector \( q \in \mathbb{R}^M \) from a Dirichlet distribution \( q \sim \mathrm{Dir}(\alpha p) \), where \( p \) denotes the global class distribution and \( \alpha \) controls the degree of non-IID variation. Smaller values of \( \alpha \) yield more skewed, domain-focused partitions (i.e., clients specializing in a small subset of subjects), whereas larger values produce more balanced distributions. We consider three heterogeneity levels with \( \alpha \in \{1.0, 10, 100\} \), covering a wide spectrum from highly specialized to near-homogeneous splits. Examples of the resulting client-domain distributions are shown in Figure~\ref{fig:data_distribution_heatmap}.

\textbf{AQUA-RAT and GSM8K.}
These datasets do not provide explicit category labels. In line with prior work, we partition them using random splits across clients. Although these tasks do not include domain metadata, the federated setting still benefits from distributional diversity introduced by varying client sample sizes and reasoning styles.
\subsection{Implementation Details}
\label{app:impl}
For all experiments involving FERA and its variants, we fix the number of clients to three. We evaluate these algorithms on the MMLU-Pro, GSM8K, and AQUA-RAT benchmarks, covering both reasoning and standard QA tasks. On the client side, we use Qwen3-4B and Llama3.1-8B as base models. During prediction, we set the sampling parameters to $top\_p = 0.8$ and temperature $= 0.3$. On the server side, no LLM is deployed for standard QA tasks. For complex reasoning tasks, the default server model is set to GPT-4o-mini.  

For LLM-Debate~\cite{du2023improving}, we ensure a fair comparison with FERA by adopting the same client models and the same number of clients as in the FERA setup. Moreover, the summarization model in LLM-Debate is configured to match the server model used in FERA, and the number of debate rounds is set to 5.

For FedAvg \citep{mcmahan2017communication}, We implement FedAvg following the OpenFedLLM framework \cite{ye2024openfedllm} on two LLMs ensuring consistency with other baselines. The training process consists of 50 communication rounds involving three clients, with data partitioned according to
to a Dirichlet distribution. Each client fine-tunes the model locally using a batch size of 16, a sequence length of 256,
and one gradient accumulation step, with a learning rate of 1e-5. To improve parameter efficiency, Low-Rank Adaptation
(LoRA) is applied and 8-bit quantization \cite{hu2022lora}. After each local update, model weights are aggregated using FedAvg on a central server, which then redistributes the updated global model to clients for the next round of local fine-tuning.

For FLora \citep{wang2024flora}, we conduct experiments using Llama-3.1-8B and Qwen3-4B in a heterogeneous setting, where clients are assigned different LoRA rank values—\emph{i.e.}, [64, 32, 16]—following the default configuration in their official implementations. All hyperparameters are kept unchanged, except for the number of clients. The models are trained using our client data and evaluated on our server data.

\subsection{Communication and Computational Cost}
\label{app:cost}

Transmission cost refers to the communication overhead between clients and the server, measured in total bits exchanged across different algorithms. For FERA, its variants, and LLM-Debate, we set the number of iterative update rounds to six. For traditional federated learning baselines such as FedAvg, the number of communication rounds is set to 50, consistent with the experimental setup described in Appendix~\ref{app:exp_setup}.

Both questions and responses are tokenized, with each response capped at a maximum of $C = 256$ tokens. The total number of queries is fixed at $M = 70$ across all benchmarks. Transmission cost accounts for both the tokens sent by clients and the server in each round, allowing for a fair comparison of communication efficiency across different methods.

\textbf{Computational cost.} We measure computational cost in FLOPs. For a transformer with $|\theta|$ parameters, one forward pass on a sequence of $n$ tokens costs approximately $2n|\theta|$ FLOPs, and one backward pass costs approximately $4n|\theta|$ FLOPs~\citep{kaplan2020scaling}. We use the following notation: $L$ is the number of clients, $M$ the number of server queries, $N$ the number of local examples per client, $n_s$ the sequence length per LLM call, $|\theta|$ and $|\theta^S|$ the number of parameters of the client and server model respectively, $K$ the number of rounds for training-free methods, $K_{\text{fed}}$ the number of rounds for training-based methods, $E$ the number of local training epochs, $B$ the training batch size, $r$ the LoRA rank, $d$ the model hidden dimension, and $\hat{L}$ the number of weight matrices with LoRA applied.

\textbf{FERA.} In each round, every client performs inference on $N$ local examples (local refinement) and $M$ server queries (client labeling), and the server aggregates responses for $M$ queries. All operations are forward-only:
\begin{equation}
\Phi_{\text{FERA}} = K \cdot \Big[\underbrace{L(N + M) \cdot 2 n_s |\theta|}_{\text{client inference}} + \underbrace{M \cdot 2 n_s |\theta^S|}_{\text{server aggregation}}\Big].
\end{equation}

\textbf{LLM-Debate.} In each round, every client generates responses for $M$ queries, the server summarizes the responses, and all clients revise based on the summary. All operations are forward-only:
\begin{equation}
\Phi_{\text{Debate}} = K \cdot \Big[\underbrace{LM \cdot 2 n_s |\theta|}_{\text{client generate}} + \underbrace{M \cdot 2 n_s |\theta^S|}_{\text{server summarize}} + \underbrace{LM \cdot 2 n_s |\theta|}_{\text{client revise}}\Big].
\end{equation}

\textbf{FedAvg.} Each of $L$ clients performs $E$ epochs of fine-tuning over $N$ local examples with batch size $B$. Each gradient step involves one forward and one backward pass:
\begin{equation}
\Phi_{\text{FedAvg}} = K_{\text{fed}} \cdot L \cdot \left\lceil \frac{EN}{B} \right\rceil \cdot 6 B \, n_s \, |\theta|.
\end{equation}

\textbf{FLoRA.} Same structure as FedAvg but with LoRA: the forward pass traverses the full model ($2Bn_s|\theta|$), while the backward pass updates only the adapter parameters $|\theta_{\text{LoRA}}| = 2\hat{L}rd \ll |\theta|$:
\begin{equation}
\Phi_{\text{FLoRA}} = K_{\text{fed}} \cdot L \cdot \left\lceil \frac{EN}{B} \right\rceil \cdot \big(2Bn_s|\theta| + 4Bn_s|\theta_{\text{LoRA}}|\big).
\end{equation}

Table~\ref{tab:cost_concrete} instantiates these formulas. We use $n_s {=} 256$ for all methods to ensure a fair comparison. For training-based methods, we use $K_{\text{fed}}{=}50$ communication rounds (as in our experimental setup), while FERA and LLM-Debate use $K{=}6$ rounds.


The table reveals that FERA's computational cost ($8.9 \times 10^{16}$ FLOPs) is dramatically lower than training-based methods: $8\times$ cheaper than FedAvg ($7.4 \times 10^{17}$) and $3\times$ cheaper than FLoRA ($2.5 \times 10^{17}$), because training-based methods require $K_{\text{fed}}{=}50$ rounds of expensive backward passes while FERA uses only forward inference over $K{=}6$ rounds. Compared to LLM-Debate ($4.1 \times 10^{16}$), FERA costs approximately $2\times$ more FLOPs due to the local refinement step ($LN$ inference calls per round), but this additional computation directly translates into substantial accuracy improvements: as shown in Section 5.3, the local refinement step enables FERA to progressively integrate client knowledge across rounds, consistently outperforming LLM-Debate by a significant margin across all benchmarks. This represents a favorable trade-off: a moderate $2\times$ increase in FLOPs yields substantial accuracy gains, while remaining an order of magnitude cheaper than training-based alternatives.
\subsection{Demonstration Selection} 
\label{app:demo_select}
In \textbf{Step 2} and \textbf{Step 3} of Algorithm~\ref{alg:FERA}, FERA selects demonstrations either from the server's query set or from the local client dataset. These demonstrations are used to facilitate client-side labeling and to update the server's query--answer set. The selection process is governed by Equation~\ref{eq:mmr}, which implement a similarity--diversity trade-off strategy. In both equations, the function \( \text{Sim} \) measures the similarity between reasoning examples. To compute this, each example is first embedded using the \texttt{paraphrase-MiniLM-L6-v2} model~\citep{reimers2019sentence}, which converts the textual reasoning into fixed-length vectors. Cosine similarity is then applied by default to quantify the similarity between embedded representations. The function \( \text{Div} \) captures the diversity within the selected set of demonstrations, encouraging the inclusion of examples that are distinct from one another. This prevents redundancy and promotes a richer representation of reasoning styles. A trade-off exists between selecting highly relevant demonstrations (those similar to the target query) and ensuring sufficient diversity within the set. This trade-off is controlled by the hyperparameter \( \lambda \), which balances the weight between relevance and diversity. Unless otherwise specified, we set \( \lambda = 0.5 \) as the default value.

\subsection{Uncertainty Calculation}
\label{app:unc_calc}
We use the uncertainty scores to guide the aggregation of client responses on the server side. Each client generates both predictions and their corresponding logits, from which we calculate uncertainty following the approach of \cite{duan2023shifting,farquhar2024detecting,zhang2025token}. To quantify the model's confidence in its generated responses, we employ a token-level entropy-based uncertainty estimation approach. Specifically, for each generated token position $t$, we calculate the entropy of the probability distribution over the vocabulary: 
\begin{equation}
H_t = -\sum_{i=1}^V p_{t,i} \cdot \log(p_{t,i} + \varepsilon),
\end{equation}
where $p_{t,i}$ represents the softmax probability of token $i$ at position $t$, $V$ is the vocabulary size, and $\varepsilon = 1 \times 10^{-10}$ ensures numerical stability. The overall uncertainty score is computed as the average entropy across all $T$ generated tokens: 
\begin{equation}
U = \frac{1}{T} \sum_{t=1}^T H_t.
\end{equation}
This metric captures the model's predictive confidence, where higher entropy values indicate greater uncertainty in token selection, while lower entropy values reflect more confident probability distributions. This approach provides a straightforward yet effective measure of generation uncertainty that can be computed directly from the model's output logits without requiring additional training or calibration, making it well-suited for our federated learning framework.

\section{Supplementary Experiments}
\subsection{Main Results Using Qwen3-4B as the Client Model}
\label{app:exp_qwen}
Figure~\ref{fig:Main_results_Qwen} and Figure~\ref{fig:Main_results_2_qwen} report the performance of FERA with Qwen3-4B as the client model, compared against FERA variants and baseline methods on the MMLU and AQUA-RAT benchmarks. The results show that FERA consistently outperforms competing approaches even when using Qwen3-4B, demonstrating its robustness across different client models.

Figure~\ref{fig:Qwen_Rounds} illustrates how the performance of FERA, FERA-Free, and FERA-Q evolves with an increasing number of interaction rounds. As the number of iterations increases, all variants show noticeable performance gains, demonstrating the effectiveness of the iterative refinement mechanism in the proposed framework.

\begin{figure}
    \centering
\includegraphics[width=\linewidth]{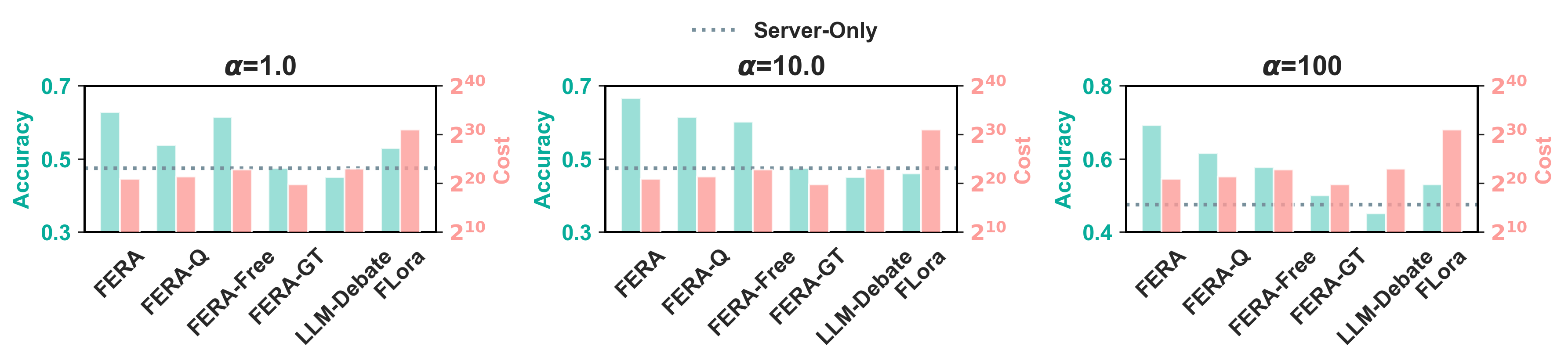}
    \caption{\footnotesize Performance comparison of \textsc{FERA} and its variants against baseline methods on the \textsc{MMLU-Pro} benchmark under varying degrees of client-level data heterogeneity, using Qwen3-4B as the base model. Client data heterogeneity is simulated via a Dirichlet distribution with concentration parameter $\alpha \in [1.0, 10, 100]$, where smaller $\alpha$ values correspond to more severe heterogeneity across clients and larger values indicate increasingly homogeneous data distributions.}
\label{fig:Main_results_Qwen}
\end{figure}

\begin{figure*}[t]
\begin{minipage}{0.44\textwidth}
    \centering
   \includegraphics[width=\linewidth]{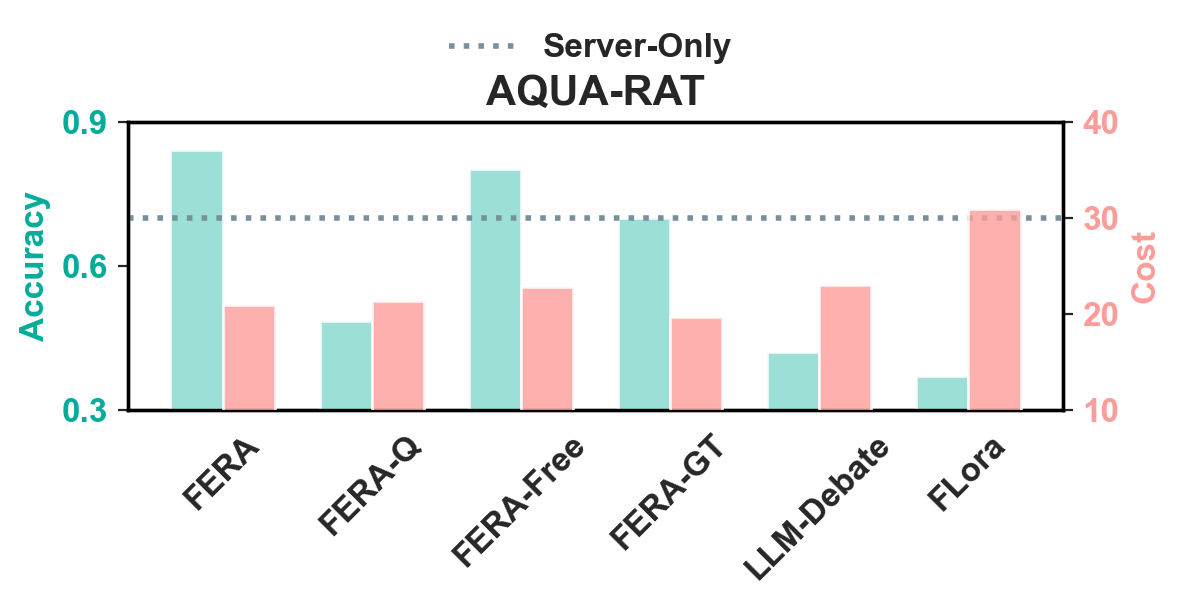}
    \caption{\footnotesize Performance comparison of FERA and its variants against baseline methods on the AQUA-RAT benchmark, with Qwen3-4B serving as the client-side model.}
    \label{fig:Main_results_2_qwen}
\end{minipage}
    \centering
    \begin{minipage}{0.55\textwidth}
        \centering
        \includegraphics[width=\linewidth]{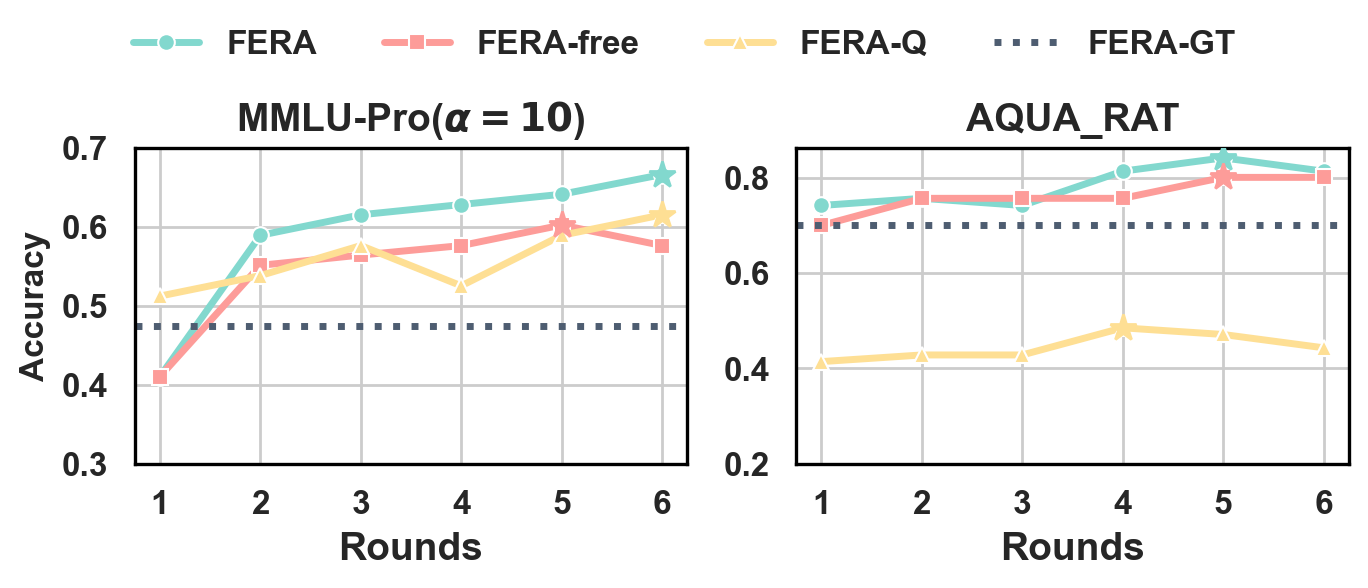}
\caption{\footnotesize Effect of interaction round count on the performance of FERA and FERA-Free in the MMLU-Pro benchmark. The Dirichlet concentration parameter is set to $\alpha = 10.0$ to simulate moderate client-level data heterogeneity. All experiments use Qwen3-4B as the client model.}
\label{fig:Qwen_Rounds}
    \end{minipage}
\end{figure*}

\subsection{Additional Ablation Study}
\label{app:exp_analysis}

\begin{figure*}[t]
    \begin{minipage}{0.49\textwidth}
        \centering
        \includegraphics[width=\linewidth]{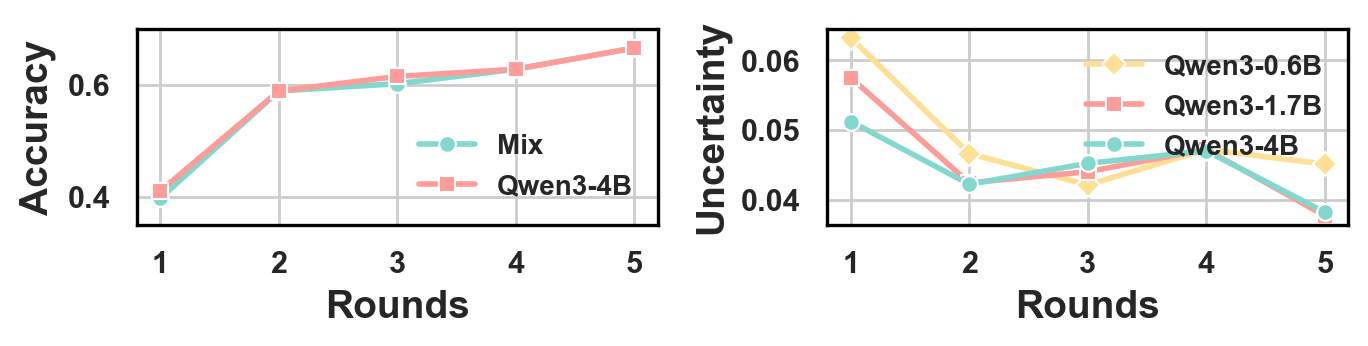}
        \vspace{-0.2in}
    \caption{\footnotesize Effect of model-capacity heterogeneity on FERA performance for the MMLU-Pro benchmark.}
    \label{fig:mix_model}
    \end{minipage}
    \begin{minipage}{0.49\textwidth}
        \centering
        \includegraphics[width=\linewidth]{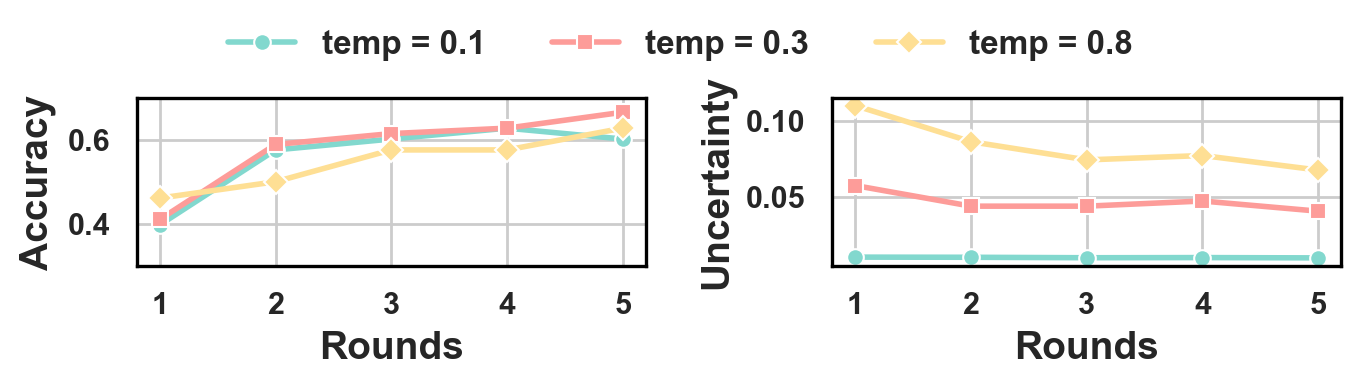}
        \vspace{-0.25in}
    \caption{\footnotesize Effect of Uncertainty Characteristics on FERA performance for the MMLU-Pro benchmark.}
    \label{fig:diff_temp}
    \end{minipage}
\end{figure*}

\begin{figure}
    \centering
    \includegraphics[width=0.8\linewidth]{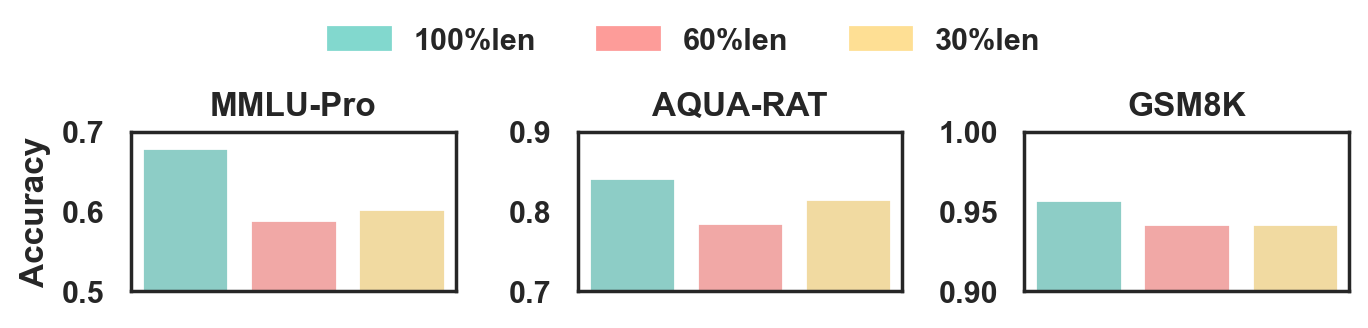}
    \caption{\footnotesize Performance of FERA under Varying Client Reasoning Response Quality.}
    \label{fig:sparse}
\end{figure}

\paragraph{Effect of Client Model Capacity Heterogeneity.} To examine the impact of heterogeneous model capacities, we conduct an ablation in which clients use LLMs of different sizes: Qwen3-4B, Qwen3-1.7B, and Qwen3-0.6B. As shown in Figure~\ref{fig:mix_model}, smaller models exhibit higher initial uncertainty but their uncertainty decreases steadily over communication rounds. The overall performance of FERA remains comparable to the homogeneous Qwen3-4B setting, suggesting that uncertainty-aware weighting effectively moderates less reliable client contributions.

\paragraph{Effect of Uncertainty Characteristics.} We isolate the effect of uncertainty from model capacity by having all clients use Qwen3-4B but with different decoding temperatures (0.3, 0.5, 0.8). As shown in Figure~\ref{fig:diff_temp}, FERA achieves similar accuracy across all settings while uncertainty decreases over rounds, suggesting that iterative refinement stabilizes generation behavior.

\paragraph{Effect of Demonstration Quality.} We simulate degraded reasoning by randomly subsampling reasoning steps (e.g., \texttt{30\%len} retains 30\% of steps). As shown in Figure~\ref{fig:sparse}, truncating demonstrations lowers FERA's performance, but degradation is graceful rather than catastrophic—uncertainty-weighted aggregation limits the influence of unreliable clients, and self-critique corrects inconsistent reasoning paths at the server.

\begin{figure*}
\centering
    \begin{minipage}{0.45\textwidth}
        \centering
        \includegraphics[width=\linewidth]{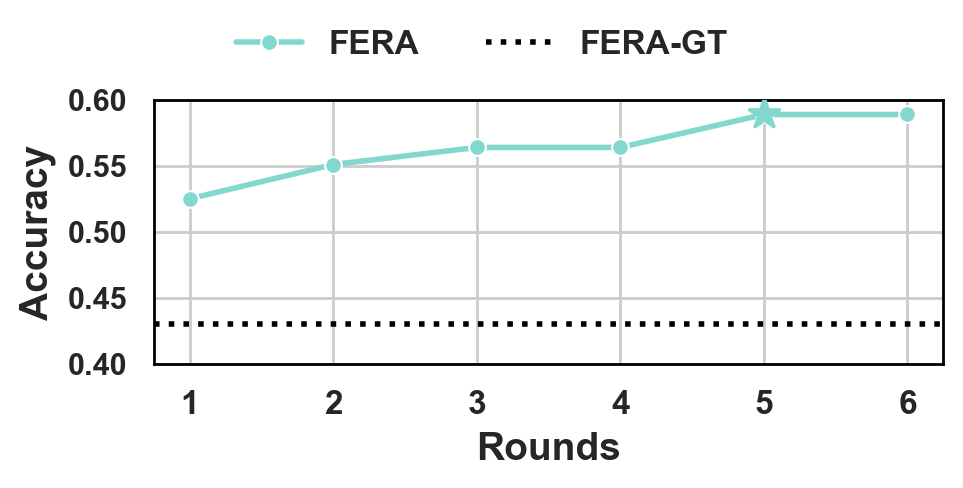}
    \caption{\footnotesize FERA performance on the MMLU-Pro reasoning benchmark. Both the client and server models are GPT-4o-mini.}
    \label{fig:client_model}
    \end{minipage}
        \begin{minipage}{0.45\textwidth}
        \centering
        \includegraphics[width=\linewidth]{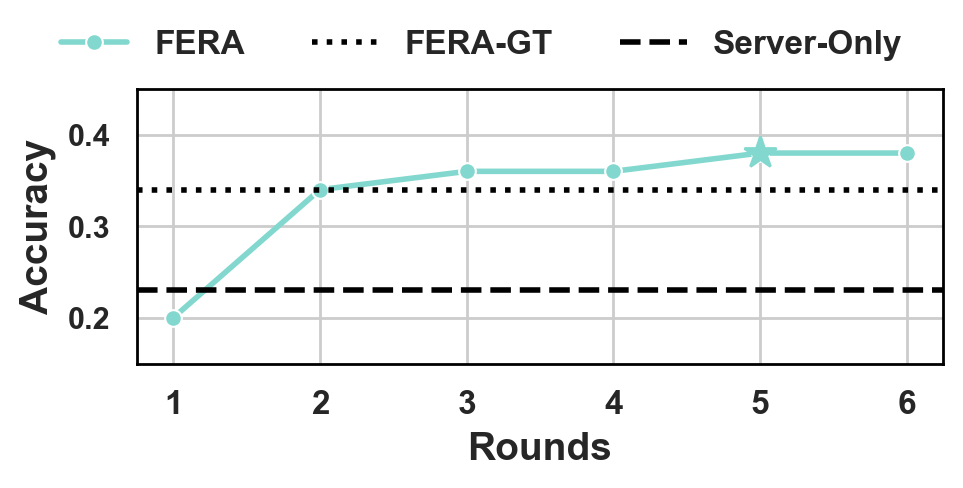}
    \caption{\footnotesize FERA performance in a specialized-domain setting on the MMLU-Pro law category. Client models use Qwen3-4B, and the server model is GPT-4o-mini. }
    \label{fig:specialized_Domain}
    \end{minipage}
\end{figure*}

\begin{figure}
    \centering
    \includegraphics[width=0.7\linewidth]{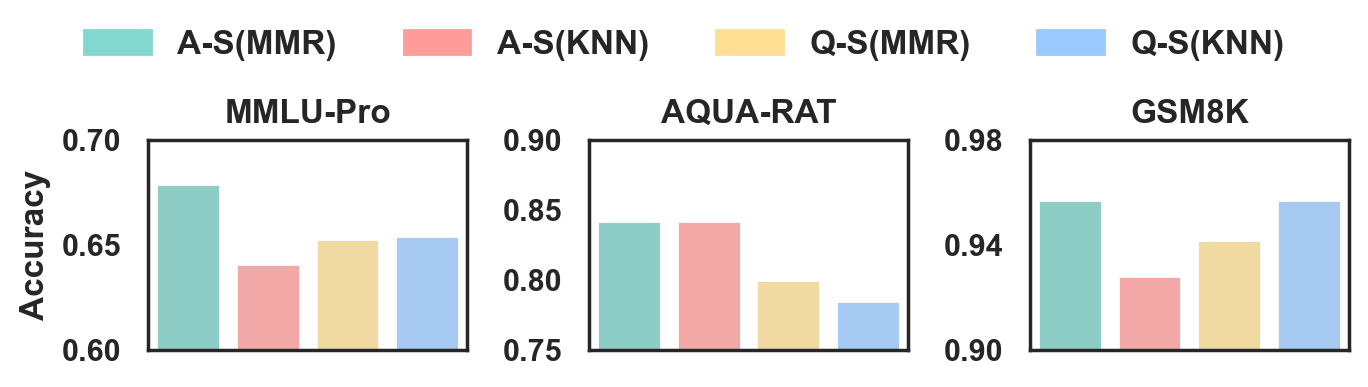}
    \caption{\footnotesize Performance of FERA under different demonstration selection strategies across MMLU-Pro, AQUA-RAT, and GSM8K.}
    \label{fig:abl_demo_select}
\end{figure}

\paragraph{Effect of Client Model Capacity.}  In our default configuration, clients run open-source LLMs (e.g., Qwen3-4B or Llama-3.1-8B), while the server model varies by task. In this ablation, to isolate the impact of using a closed-source stack and to remove model heterogeneity between endpoints, we set both the client and the server to GPT-4o-mini. The GPT-4o-mini API exposes response-level scores (log-probability–based confidences), which we use directly to compute uncertainty for our Uncertainty-Aware Demonstration Selection and Uncertainty-Aware Aggregation, without any additional reward/critic model. This keeps the uncertainty signal aligned with the generator’s own beliefs and ensures a fair comparison under identical decoding settings. The results, presented in Figure~\ref{fig:client_model}, demonstrate that FERA remains effective even in a closed-source setting. It consistently outperforms the FERA-GT baseline across communication rounds, indicating that FERA is model-agnostic and capable of leveraging native confidence signals to enhance performance.

\paragraph{Effect of Specialized Domains.}
\label{app:spec_domain}
To evaluate FERA in settings where client expertise is concentrated within a particular domain, we conduct an additional experiment using the \texttt{law} category of MMLU-Pro, which is the domain where the server model shows the weakest standalone performance. We construct a server-side evaluation set by selecting 50 law questions, and distribute all remaining MMLU-Pro questions across clients using the same non-IID Dirichlet partitioning as in our main experiments. This yields a domain-imbalanced configuration in which the server has limited direct exposure to law-domain data and must rely on clients that hold the majority of the relevant information. As shown in Figure~\ref{fig:specialized_Domain}, FERA steadily improves over communication rounds and outperforms all baselines, indicating that the framework can effectively integrate specialized client knowledge even when the server begins with weak domain proficiency.

\paragraph{Effect of Demonstration Selection Strategy. } Figure~\ref{fig:abl_demo_select} presents a comparison of four client-side demonstration selection strategies, which guides how to determine the set $\cS_{k, Q}^i$ denoted in Step 2 and 3 of FERA. They are (i) \textit{A-S (MMR)}, our default selection (ii) \textit{A-S (KNN)} uses the same adaptive selector but instantiates it with \(k\)-nearest neighbors (KNN). Given a sample \((q, s_{1:T}, a)\) and a candidate dataset, the selector first embeds the sample, computes similarity scores to all candidates, and selects the top-\(k\) nearest neighbors as demonstrations. Unlike \textit{A-S (MMR)}, which explicitly trades off relevance and diversity, \textit{A-S (KNN)} selects demonstrations based solely on similarity and does not encourage diversity among the chosen examples. (iii) \textit{Q-S (MMR)}, which applies MMR based solely on question similarity; and (iv) \textit{Q-S (KNN)}, which uses KNN based only on question similarity. Among these, \textit{A-S (MMR)} achieves the highest accuracy, outperforming both its KNN-based variant and the question-only baselines. This performance gain highlights the effectiveness of combining MMR's relevance-diversity trade-off with the adaptive selector’s incorporation of broader contextual signals beyond simple question similarity.

\paragraph{Effect of Demonstration Quantity.}  We investigate the impact of the number of context demonstrations on performance of FERA and FERA-Q by comparing setups with 1, 3, and 5 examples. As shown in Figure~\ref{fig:demo_num},~\ref{fig:demo_num2} increasing the number of context examples improves the LLM’s understanding of the query, resulting in higher response accuracy, which will also lead to the better performance of FERA and FERA-Q.

\paragraph{Effect of Varying Client Numbers.}  We investigate the impact of client population size by comparing the performance of FERA and FERA-Q under configurations with 3 and 5 clients. As shown in Figures~\ref{fig:client_num} and~\ref{fig:client_num2}, increasing the number of clients leads to a noticeable decline in performance. The primary cause is that, with a fixed-size dataset partitioned across more clients, each client receives fewer and less diverse demonstrations, which increases data heterogeneity across clients. This heightened heterogeneity leads to more divergent and conflicting reasoning paths for the same query, making server-side aggregation more challenging.

Importantly, FERA does not require all clients to participate in every round. In large-scale scenarios, the server can select a subset of $L' < L$ clients per round, controlling the level of heterogeneity while still benefiting from distributed data. This partial participation strategy directly mitigates the performance degradation observed above, since fewer clients per round reduces conflicting reasoning paths while the iterative refinement process can still incorporate different clients across rounds. Combined with uncertainty-aware weighting, which naturally assigns lower influence to noisy or irrelevant clients, FERA's architecture is compatible with larger federations without fundamental modifications.

\begin{figure*}[t]
    \begin{minipage}{0.58\textwidth}
        \centering
        \includegraphics[width=\linewidth]{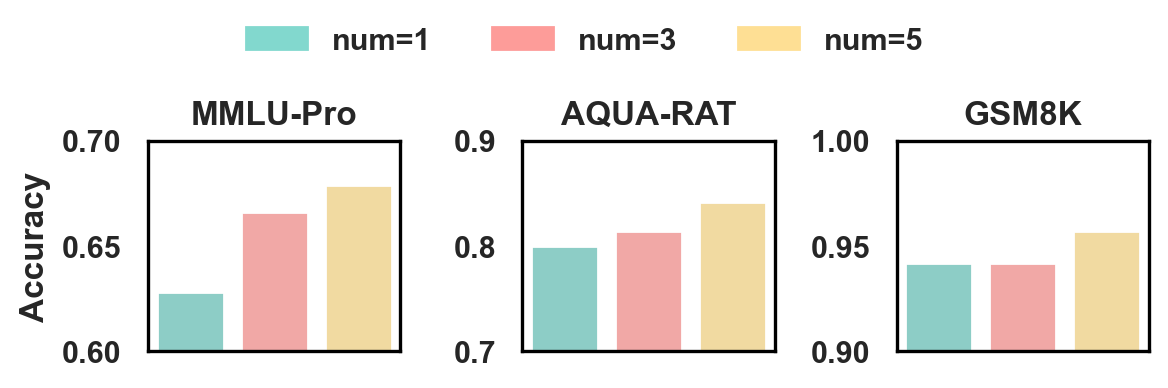}
    \caption{\footnotesize Effect of in-context demonstration quantity on FERA performance for differnt benchmarks.}
    \label{fig:demo_num}
    \end{minipage}
    \centering
        \begin{minipage}{0.39\textwidth}
        \centering
        \includegraphics[width=\linewidth]{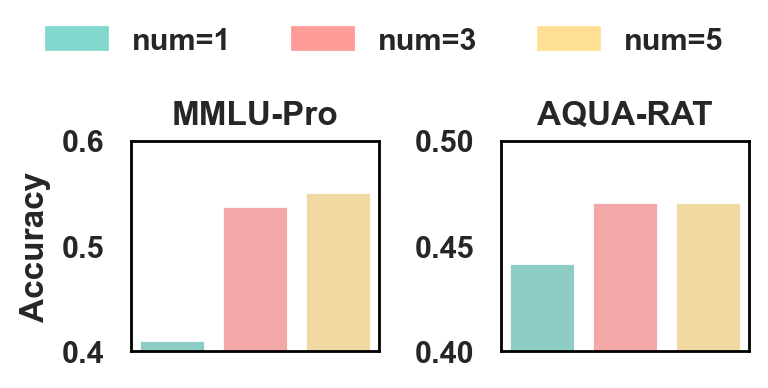}
    \caption{\footnotesize Effect of the number of demonstrations on FERA-Q performance for MMLU-Pro and AQUA-RAT benchmarks.}
    \label{fig:demo_num2}
    \end{minipage}
        \begin{minipage}{0.58\textwidth}
        \centering
        \includegraphics[width=\linewidth]{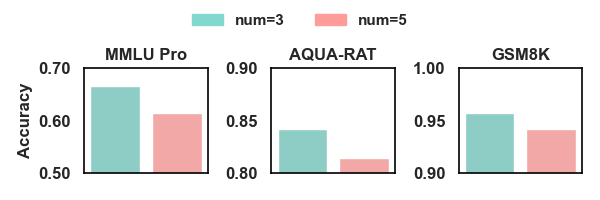}
    \caption{\footnotesize Effect of number of clients on FERA performance for different benchmarks.}
    \label{fig:client_num}
    \end{minipage}
    \centering
        \begin{minipage}{0.39\textwidth}
        \centering
        \includegraphics[width=\linewidth]{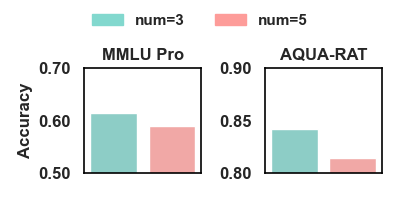}
    \caption{\footnotesize Effect of the number of clients on FERA-Q performance for MMLU-Pro and AQUA-RAT benchmarks.}
    \label{fig:client_num2}
    \end{minipage}
\end{figure*} 
\section{Design Considerations for the Uncertainty Measure}
\label{app:unc_design}

In the federated reasoning setting, clients operate under strict privacy, computational, and communication constraints, which make many uncertainty-estimation techniques used in centralized LLM deployments impractical. Approaches such as semantic uncertainty, model-specific confidence scores, or auxiliary verifier–based methods generally require multiple forward passes, additional embedding or classifier modules, or model-dependent calibration procedures. These operations introduce substantial inference overhead, increase communication cost, and conflict with FERA’s training-free and model-agnostic design objectives. Token-level entropy, by contrast, offers a lightweight and consistent alternative. It is computed directly from the token-probability distributions already produced during decoding, requires no auxiliary components or calibration, and adds no additional computational or architectural assumptions for the client.

To evaluate whether more complex alternatives offer meaningful advantages, we conducted a small comparative study. Ten questions were sampled from the server’s evaluation set, and for each of three clients we computed both token-level entropy and semantic-uncertainty scores, while also recording the computation time required for each method. The token-level uncertainty score is computed as described in Section~\ref{app:unc_calc}, and the semantic-uncertainty measure follows the formulation in \cite{kuhn2023semantic}. The results, summarized in Figure~\ref{fig:uncertainty_comparison}, show that the two uncertainty measures are closely aligned in magnitude across all clients, indicating that semantic uncertainty does not yield qualitatively different signals. In contrast, its computational cost is substantially higher—typically 4–5× slower—due to additional encoding and comparison steps. These observations reinforce our design choice: token-level entropy provides a reliable, efficient, and privacy-compatible uncertainty signal that aligns with the constraints and objectives of federated reasoning systems like FERA.
\begin{figure}
    \centering
\includegraphics[width=0.5\linewidth]{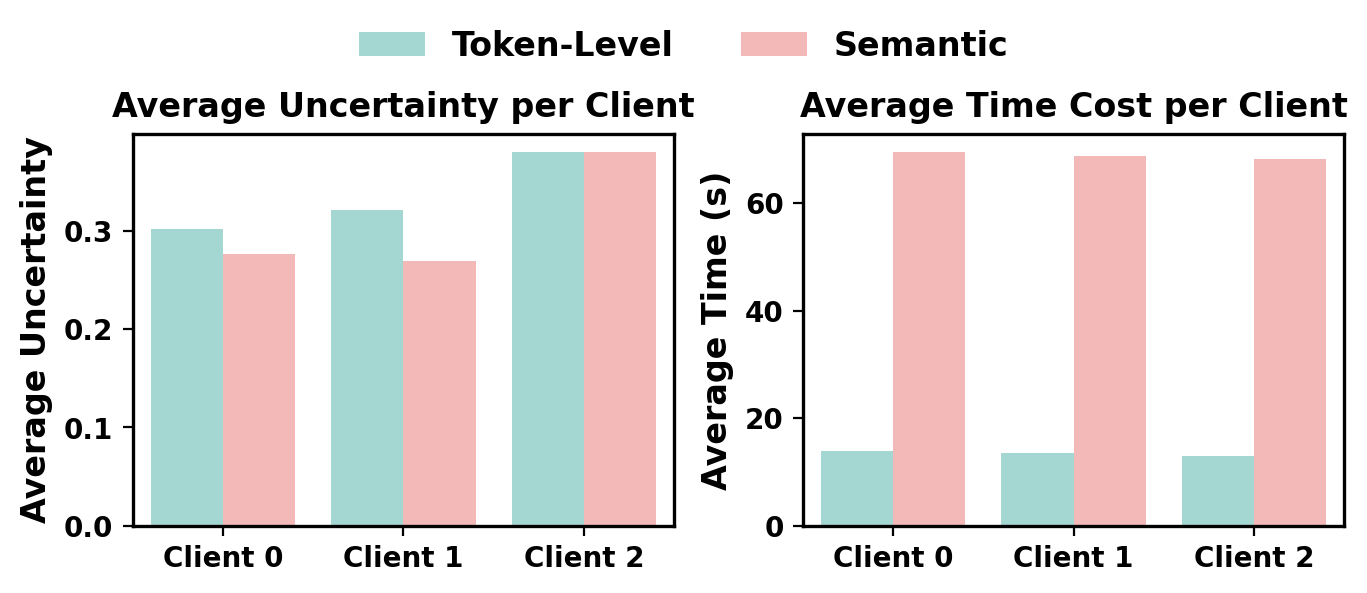}
\caption{Comparison of token-level and semantic uncertainty.
\textbf{Left}: Average uncertainty values computed for ten sampled queries across three clients. \textbf{Right:} Average computation time per client.}
\label{fig:uncertainty_comparison}
\end{figure}

\section{Privacy Analysis of the FERA Framework Details}
\label{app:privacy}
The main text analyzes FERA’s robustness against prompt extraction attacks. In this section, we present the GPT-4o prompt used for demonstration reconstruction, along with real-world cases illustrating how FERA responds under such attacks. The results are shown in Figure~\ref{fig:Privacy_Q_A}.

To further assess concerns regarding client-side privacy, particularly whether client-generated responses may inadvertently disclose private information beyond the server-issued query, we conducted additional experiments on the MMLU-Pro benchmark. Specifically, we simulated a federated setup with three clients using LLaMA3-8B-Instruct under a Dirichlet partitioning scheme with concentration parameter $\alpha = 10$. We then analyzed the generated CoTs for the presence of personal identifiers.

To quantify potential privacy leakage, we applied the “bert-base-NER” model \citep{devlin2019bert} to both the original server prompts and the client-generated reasoning response. Following the methodology of \citet{edemacu2025privacy}, we measured how often response contained personal identifiers not already present in the original prompts. The results, summarized in Table~\ref{tab:privacy-leakage}, show that only a small fraction of response included such identifiers, indicating that FERA poses minimal risk of unintended privacy leakage in this setting.

\begin{tcolorbox}[colback=lightblue!10, colframe=lightblue!90!black, title={GPT-4o Prompt for Reconstruct ICL Examples}, breakable]
\hypertarget{GPT-4o Prompt for Reconstruct}{}
\small\ttfamily
response = openai.ChatCompletion.create(\\
\hspace*{1em}model="gpt-4o",\\
\hspace*{1em}messages=[\\
\hspace*{2em}\{\\
\hspace*{3em}"role": "system",\\
\hspace*{3em}"content": "You are a helpful assistant that specializes in natural language generation and in-context learning design."\\
\hspace*{2em}\},\\
\hspace*{2em}\{\\
\hspace*{3em}"role": "user",\\
\hspace*{3em}"content": (\\
\hspace*{4em}"I am using in-context learning to answer a question, where the context is composed of several relevant question--CoT answer pairs. I will provide the query and its corresponding final answer.Please assist in reconstructing a set of five question-CoT answer pairs that are topically and logically aligned with the final answer, serving as effective context examples."\\\\
\hspace*{4em}"Question: XXXXX"\\
\hspace*{4em}"Answer: XXXXX"\\
\hspace*{3em})\\
\hspace*{2em}\}\\
\hspace*{1em}]\\
)

\end{tcolorbox}
\begin{table}[h]
    \centering
    \begin{tabular}{lcccc}
        \hline
        & Client 1 & Client 2 & Client 3 & Avg. \\
        \hline
        Value (\%) & 0.031 & 0.049 & 0.051 & 0.043 \\
        \hline
    \end{tabular}
    \caption{Fraction of client-generated reasoning response containing personal identifiers, as detected by the \texttt{bert-base-NER} model.}
    \label{tab:privacy-leakage}
\end{table}

\begin{figure*}[h]
    \centering
    \includegraphics[width=\linewidth]{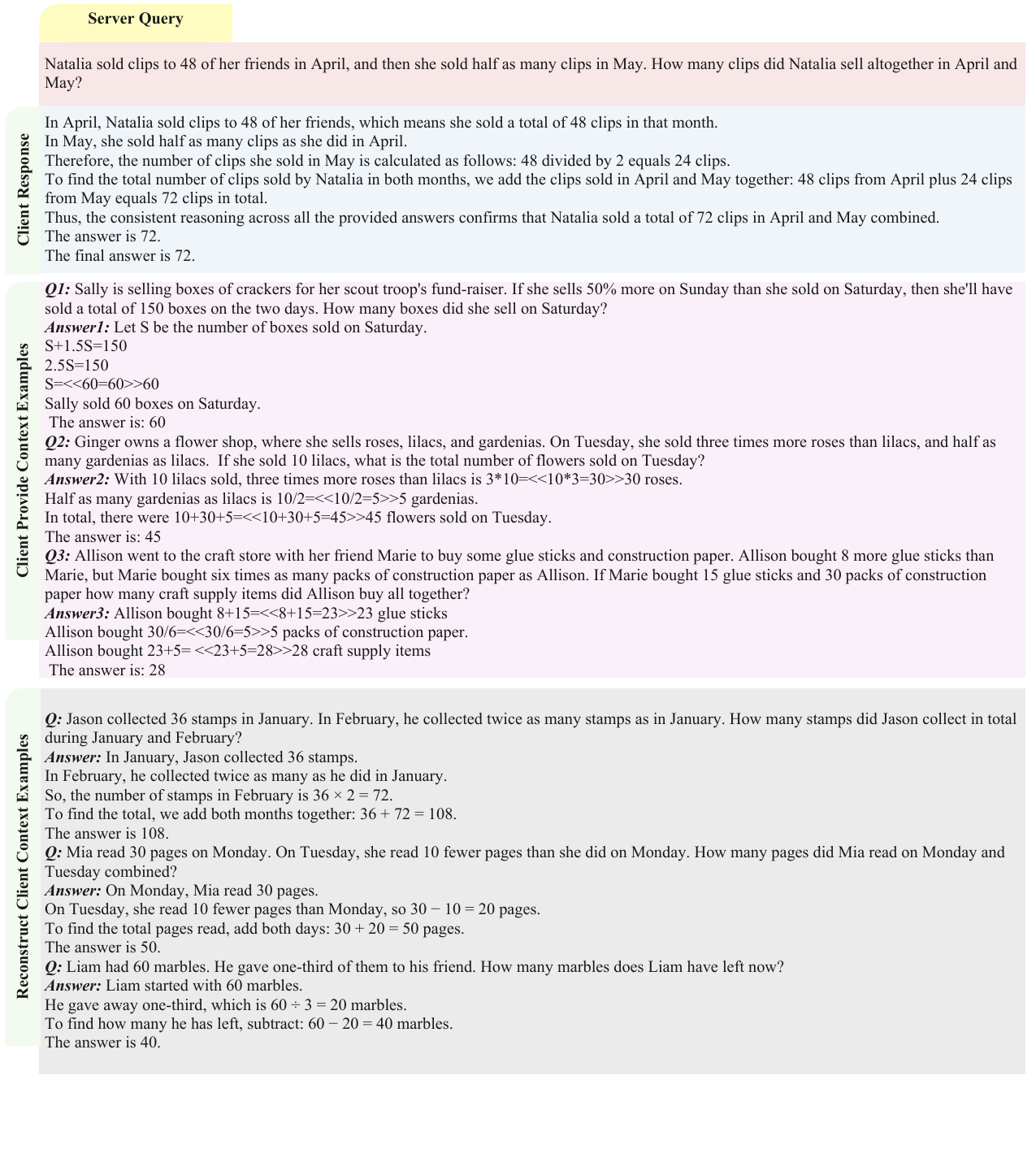}
    \caption{Privacy Analysis}
    \label{fig:Privacy_Q_A}
\end{figure*}

\section{Case Study}
\label{app:case_study}
In this section, we present real examples that illustrate how reasoning answers are updated across interaction rounds. These cases are shown in from Figure~\ref{fig:Q1_a1} to Figure~\ref{fig:Q3_a3}.

\begin{figure*}[h]
    \centering
    \includegraphics[width=\linewidth]{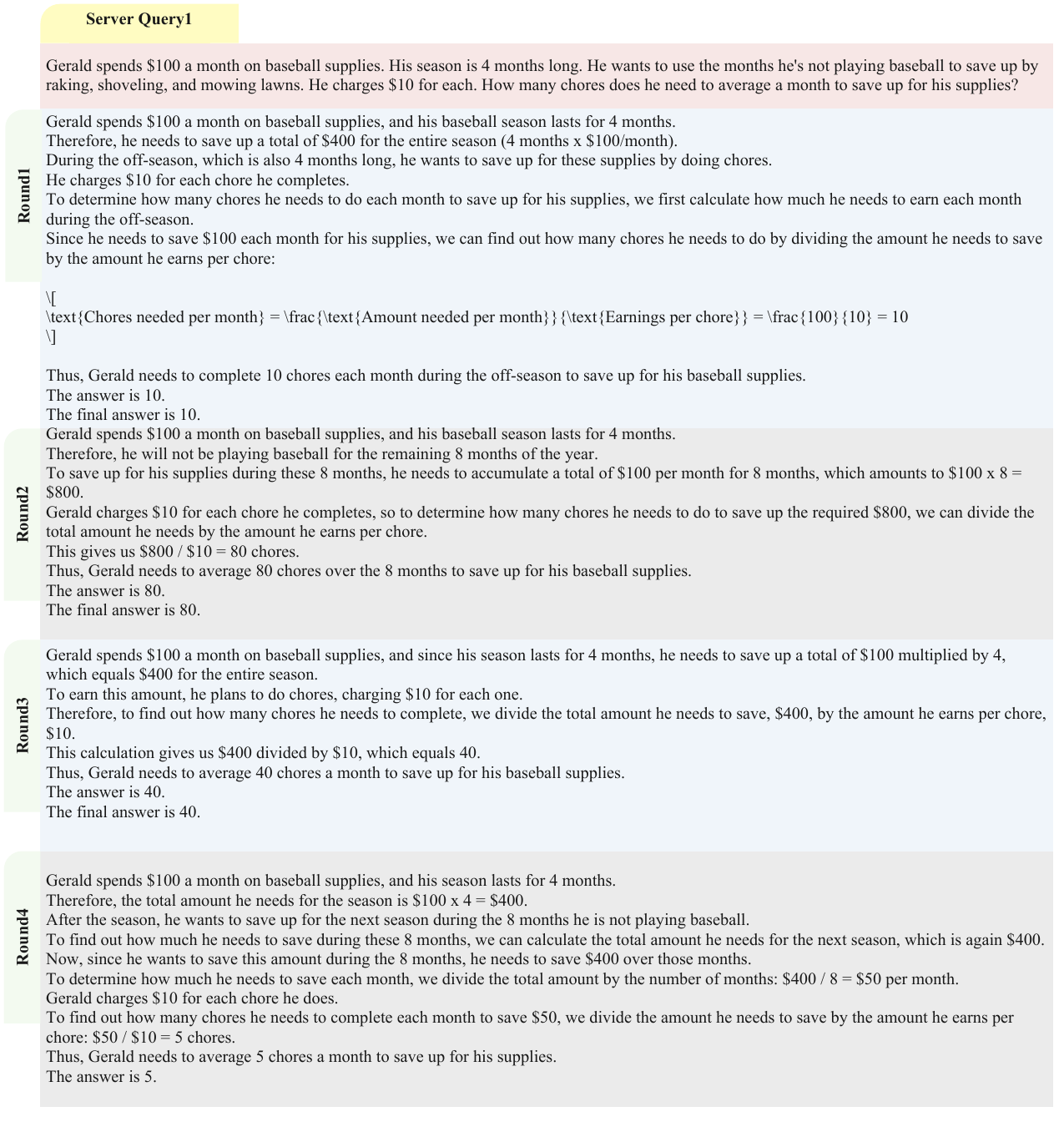}
    \caption{Example Reasoning Answer iterative update for Query 1}
    \label{fig:Q1_a1}
\end{figure*}
\begin{figure*}[h]
    \centering
    \includegraphics[width=\linewidth]{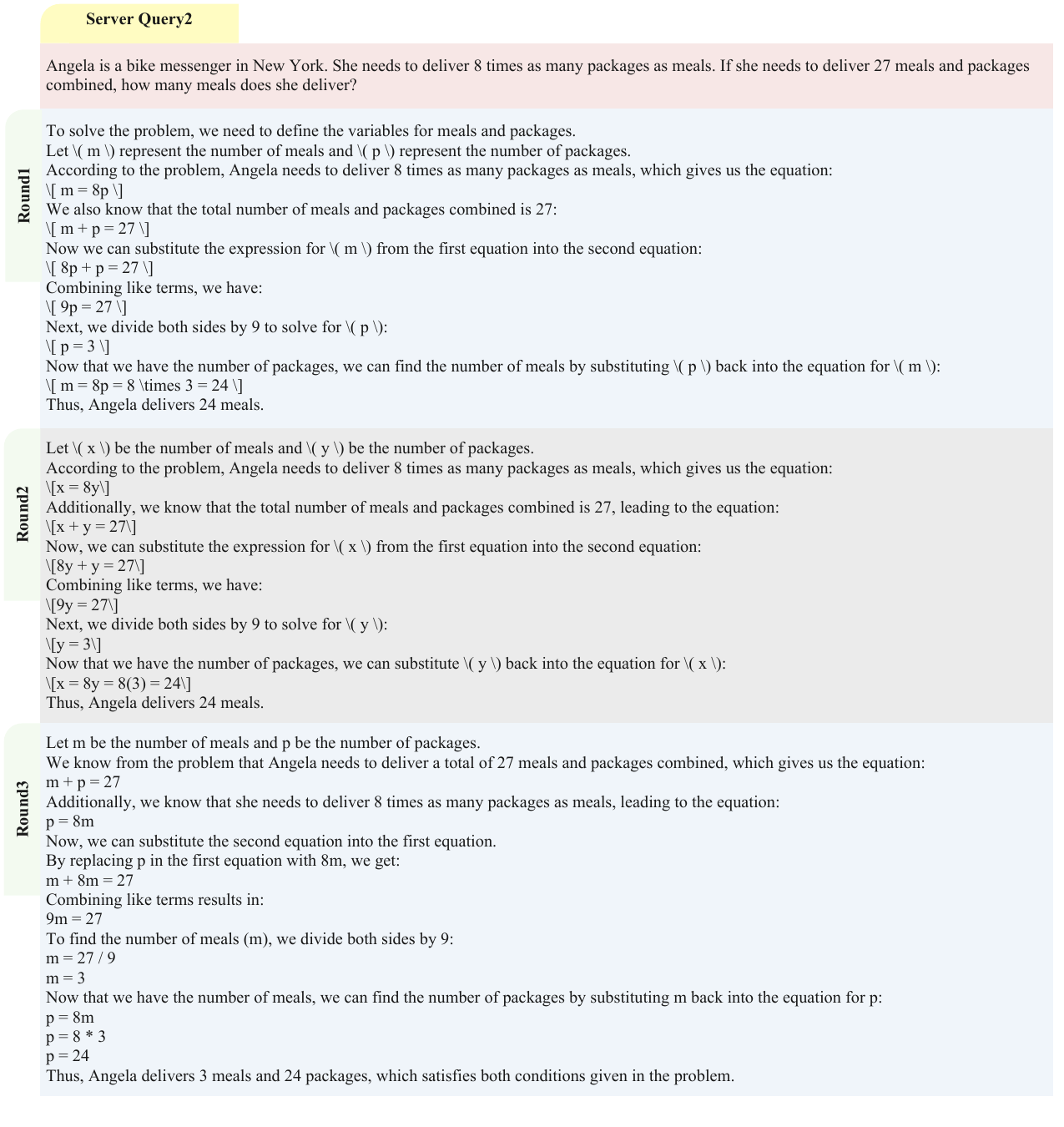}
    \caption{Example Reasoning Answer iterative update for Query 2}
    \label{fig:Q2_a2}
\end{figure*}
\begin{figure*}[h]
    \centering
    \includegraphics[width=\linewidth]{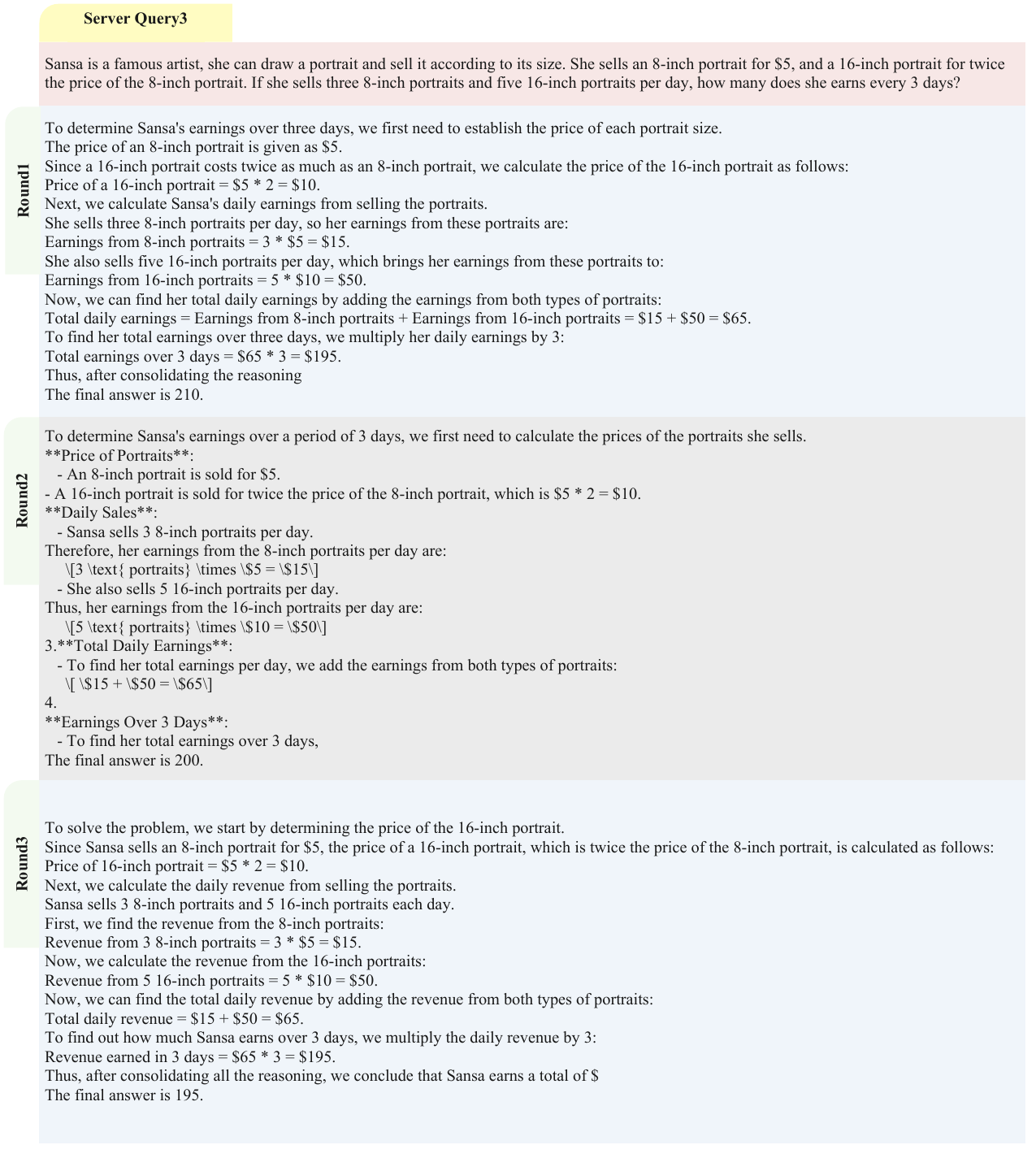}
    \caption{Example Reasoning Answer iterative update for Query 3}
    \label{fig:Q3_a3}
\end{figure*}
\end{document}